\documentclass{article}

\usepackage[main, final]{neurips_2026}

\usepackage[utf8]{inputenc} 
\usepackage[T1]{fontenc}    
\usepackage{hyperref}       
\usepackage{url}            
\usepackage{booktabs}       
\usepackage{amsfonts}       
\usepackage{nicefrac}       
\usepackage{microtype}      
\usepackage{xcolor}         
\usepackage{subfiles}
\usepackage{amsmath}
\usepackage{graphicx}
\usepackage{wrapfig}
\usepackage{enumitem}
\usepackage[font=small,labelfont=bf]{caption}
\usepackage{booktabs}
\usepackage{multirow}
\usepackage{makecell}
\usepackage{url}            
\hypersetup{
    colorlinks=true,
    linkcolor=blue, 
    citecolor=blue,
    urlcolor=blue
}
\usepackage{subcaption}

\newcommand{\Paragraph}[1]{\vspace{-1mm} \noindent \textbf{#1} \hspace{0mm}}

\usepackage{amsthm}
\newtheorem{lemma}{Lemma}
\newtheorem{proposition}{Proposition}
\newtheorem{assumption}{Assumption}
\newtheorem{definition}{Definition}
\newtheorem{remark}{Remark}
\newtheorem{observation}{Observation}
\newtheorem{theorem}{Theorem}
\newtheorem{corollary}{Corollary}

\title{Non-Colliding Biometric Identities for Digital Entities: Geometry, Capacity, and Million-Scale Virtual Identity Provisioning}

\author{%
  Yuyang Ji$^1$, Yixuan Shen$^1$, Anil Jain$^2$, Xiaoming Liu$^3$, Feng Liu$^1$\vspace{1mm}\\ 
  $^1$ Department of Computer Science, Drexel University\\
  $^2$ Department of Computer Science and Engineering, Michigan State University\\
  $^3$ Department of Computer Science, University of North Carolina at Chapel Hill\\
  \texttt{jain@msu.edu, liuxm@cs.unc.edu, \{yj428,ys844,fl397\}@drexel.edu} \\
}

\begin{document}

\maketitle


\begin{abstract}
Digital entities such as AI agents and humanoid robots increasingly operate alongside real humans, yet their identity infrastructure is based on credentials rather than embodied biometric identity. We introduce Biometric Identity Provisioning (BIP), a new problem and solution framework that addresses: given an enrollment gallery of real human identities, provision virtual identities that are non-colliding with every enrolled identity, maintain sufficient inter-class separability, and are realizable as high-fidelity face images. The key geometric insight is that real face identities occupy a low-dimensional subspace of the embedding hypersphere, leaving no residual subspace for virtual identities. Hence, virtual identities must instead be allocated to unclaimed gaps within the real face manifold itself. BIP is therefore a constrained packing problem: available gaps vastly exceed any foreseeable enrollment scale, and provisioned identities are long-lived and lifecycle-managed, with rare conflicts under continued enrollment resolved by revocation and reassignment. Grounded in this geometry, our repulsion-based allocation is not bounded by any fixed provisioning count; we demonstrate 10M non-colliding virtual identity embeddings against a gallery of 360K real identities. Realizing these embeddings as face images requires a generator that operates outside the training distribution of real face images; we introduce GapGen, a gap-aware generator trained with a curriculum that progressively extends synthesis into non-colliding regions, validated at 1M photorealistic virtual face images. We further construct v-LFW, a virtual counterpart to the LFW face dataset, with protocols for virtual face verification, cross-reality matching, real-vs-virtual detection, and unified recognition and detection.
\href{https://vilab-group.com/project/bip}{Project} 
\end{abstract}


\section{Introduction}\label{sec:intro}

Identity systems are being reshaped by a new class of actors: AI agents 
autonomously coordinate enterprise workflows~\cite{south2025identity}, 
humanoid robots interact face-to-face with people in service environments, 
and digital co-workers operate under zero-trust policies alongside human 
employees.
This is not a projected future: a recent industry survey reports
that 40\% of organizations already have AI agents in production,
with another 31\% running pilots or tests~\cite{csa2026securingagents}.
Microsoft has launched dedicated infrastructure to govern these agents as first-class identities~\cite{microsoft2025entralearn}.
Yet this infrastructure assigns credential identity (tokens, certificates, 
access scopes), not \emph{embodied biometric identity}.
As digital entities operate in physical spaces alongside real humans, a question arises: \textit{how do we provision digital entities with long-lived, 
unique biometric identities that do not collide with any enrolled real human 
identity?}
We term this the \textbf{Biometric Identity Provisioning (BIP)} 
problem (see Fig.~\ref{fig:teaser}).

Face is the natural entry point for biometric identity provisioning  for virtual entities: it is the modality most directly involved in human-entity interaction, the most mature in generative modeling, and the one for which large-scale enrollment galleries already exist~\cite{li2011handbook}; gait~\cite{zhang2022gait,su2024open}, body shape~\cite{liu2023learning,liu2024distilling} and whole-body traits~\cite{liu2024farsight,liu2025person,su2025hamobe,kim2025sapiensid,su2026sapiensid2} are natural extensions.
A substantial body of work has studied synthetic face generation, but with a fundamentally different objective: producing training data for face recognition models. 
%
Methods such as Arc2Face~\cite{papantoniou2024arc2face} and
Vec2Face/Vec2Face+~\cite{wu2024vec2face,wu2025vec2faceplus} condition face 
generation on real human identity embeddings, generating faces that correspond to
existing real people in embedding space; they are \emph{identity cloners},
not \emph{identity creators}.
%
Methods without this conditioning, such as DCFace~\cite{kim2023dcface}, optimize for separability among
synthetic faces, but whether generated identities collide with real enrolled humans is simply not a design concern; empirically, we show such collisions do occur at non-trivial rates.
Thus, a synthetic identity useful for training a classifier is not necessarily a biometric identity suitable to assign to a digital entity. Existing methods neither define nor jointly satisfy the requirements of high-fidelity realization, inter-class separability, and guaranteed \textbf{non-collision} with any enrolled real humans.

\begin{figure}
  \centering
  \includegraphics[width=\linewidth]{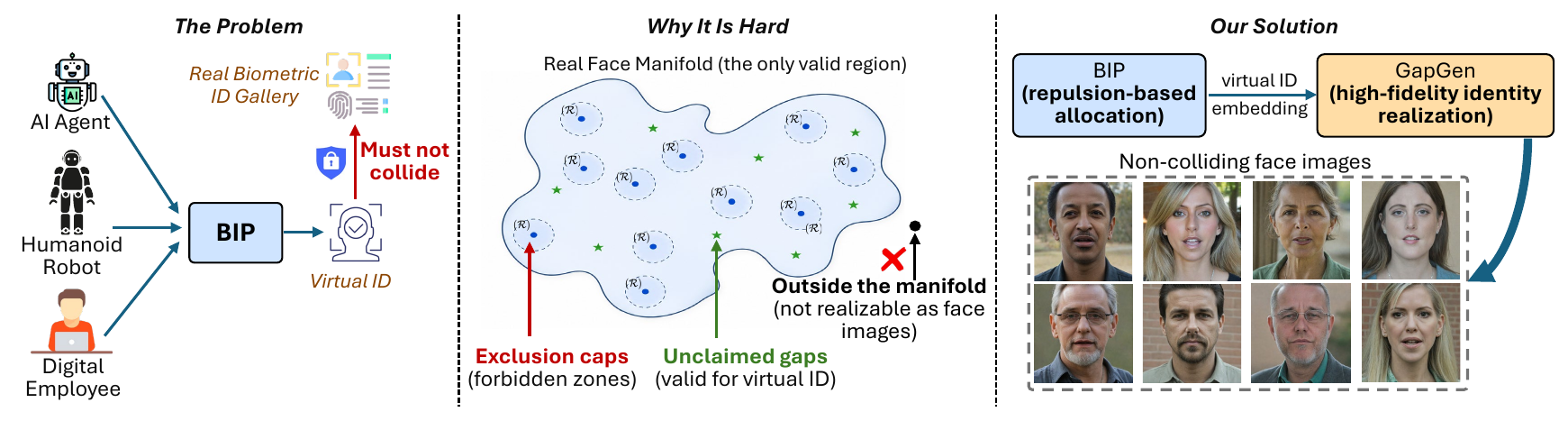}
  \caption{
\textbf{Should a digital entity have a distinctive  face?}
As AI agents, humanoid robots, and digital employees operate
alongside real humans, they require long-lived biometric
identities that must not be mistaken for any enrolled real
person (left).
This is geometrically non-trivial: real face identities occupy
a low-dimensional manifold with no residual subspace, forcing
virtual identities into the narrow unclaimed gaps between
exclusion caps; outside the manifold, embeddings cannot be
realized as natural face images (middle).
BIP provisions non-colliding virtual identity embeddings via
repulsion-based allocation, and GapGen realizes them
as high-fidelity $1024{\times}1024$ portrait faces (right).
}
  \label{fig:teaser}
\end{figure}

Hence, we propose a new formulation, termed BIP: synthetic identities shall be
treated not merely as generated images, but as allocatable positions in a
shared biometric identity space.
We define BIP as: 
\begin{center}
	\vspace*{-0.2cm}
	\setlength\fboxrule{0.0pt}
	\noindent\fcolorbox{black}[rgb]{0.95,0.95,0.95}{\begin{minipage}{0.98\linewidth}
\textit{Given a mapping from face images to $d$-dimensional
L2-normalized embeddings on $\mathbb{S}^{d-1}$
and an enrollment gallery $\mathcal{R} \subset \mathbb{S}^{d-1}$,
provision a set of virtual embeddings $\mathcal{V} \subset \mathbb{S}^{d-1}$ satisfying
non-collision with $\mathcal{R}$ at threshold $\tau$,
inter-class separability within $\mathcal{V}$,
and high-fidelity realizability.}
	\end{minipage}}
	\vspace*{-0.2cm}
\end{center}
A natural but na\"ive geometric view treats BIP as a \textbf{capacity} problem:
the number of separable positions on $\mathbb{S}^{d-1}$ at threshold
$\tau$ is arbitrarily large even within the effective face
subspace, suggesting ample room for virtual identities.
This view is misleading: 
\textbf{\emph{(1) No residual subspace exists.}}
Virtual identities satisfying BIP constraints do not occupy
arbitrary low-energy directions of the hypersphere; their PCA
energy distribution closely follows that of real identities, with
both concentrating over 95\% of their variance within the same
$k{=}269$ leading principal components out of $d{=}512$.
Valid virtual identities must occupy \emph{unclaimed gaps
within the real face manifold} rather than a geometrically separate
region.
\textbf{\emph{(2) Capacity does not equal realizability.}}
Embeddings satisfying the non-collision constraint are not
automatically realizable as high-fidelity face images:
the further a virtual identity is pushed from real identity
territories, the harder it becomes for a generative model to
generate it as a natural face.
BIP is, therefore, a constrained identity-allocation problem on the
real identity manifold, not an unconstrained sampling problem on
$\mathbb{S}^{d-1}$.

Grounded in this geometric view, we address BIP through an identity-first pipeline that separates allocation in embedding space from realization in image space:

\emph{\textbf{(1) Formalization.}}
We introduce the first formal definition of BIP and three evaluation metrics operationalizing its requirements:
non-collision, inter-class separability, and perceptual image
quality.
\emph{\textbf{(2) Geometry and allocation.}}
We characterize the geometry of the real identity
manifold and show that safe virtual identities must occupy
directional gaps within it.
We derive the \textbf{repulsion direction}, $z^*$ as a heuristic
that moves candidates away from real identity clusters, enriched
with PCA-aware noise to maintain face-manifold compatibility.
Combined with exact hard checks, the allocation strategy is not
bounded by any fixed provisioning count: the face manifold
contains vastly more unclaimed gaps than any foreseeable
enrollment gallery, and we demonstrate 10M non-colliding
virtual identity embeddings against a gallery of 360K real
identities with no observed collision.
\emph{\textbf{(3) Realization via GapGen.}}
We introduce a gap-aware generator trained with a progressive curriculum
that extends face synthesis into the non-collision region.
\emph{\textbf{(4) Benchmark.}}
We construct \textbf{v-LFW}, a virtual counterpart to the LFW face dataset with
5{,}749 provisioned identities and 13{,}233 images, and introduce protocols spanning virtual face verification, cross-reality
matching, real-vs-virtual detection, and unified recognition and
detection.
To support these protocols, we provide IAPCT (Identity-Anchored
Patch Consistency Transformer) as a lightweight diagnostic
evaluation tool.


\section{Related Work}

\Paragraph{Synthetic Face Generation for Recognition.}
%
Synthetic face data has been widely explored as a privacy-preserving
alternative to web-collected face datasets.
Early work such as SynFace~\cite{qiu2021synface} and
SFace~\cite{boutros2022sface} uses class-conditional generative models,
while DigiFace-1M~\cite{bae2023digiface} adopts graphics-based rendering.
More recent methods improve identity diversity via diffusion:
IDiff-Face~\cite{boutros2023idiff} introduces identity-conditioned
diffusion, and DCFace~\cite{kim2023dcface} disentangles identity and
style via dual-condition diffusion.
Arc2Face~\cite{papantoniou2024arc2face} and
Vec2Face and Vec2Face+~\cite{wu2024vec2face,wu2025vec2faceplus} further
generate identity- and attribute-controllable face datasets from recognition-based
features, with Vec2Face+ explicitly controlling inter-class separability,
intra-class variation, and identity consistency.
VIGFace~\cite{kim2025vigface} is especially related because it pre-assigns
virtual identities in feature space before synthesis.
However, these methods are designed primarily for privacy-friendly or
high-performing face recognition training data~\cite{liu2022cfsm}.
They do not formulate identity generation as gallery-conditioned provisioning:
whether a generated identity collides with an enrolled real human identity
is not a primary constraint, and long-lived assignment to digital entities is
outside their problem setting.

\Paragraph{Hyperspherical Identity Allocation.}
Modern face recognition systems~\cite{kim2022cluster,kim2024keypoint} map faces to normalized embeddings on a
high-dimensional hypersphere, where identity verification uses cosine
similarity or angular distance.
SphereFace~\cite{liu2017sphereface}, CosFace~\cite{wang2018cosface},
ArcFace~\cite{deng2019arcface}, and
AdaFace~\cite{kim2022adaface} introduce angular margin losses that
structure this hypersphere to improve inter-class separation and
intra-class compactness, with AdaFace further adapting the margin to
image quality via feature norms.
The most closely related work to BIP is
HyperFace~\cite{shahreza2024hyperface}, which formulates synthetic face
data generation as a packing problem on the recognition hypersphere
and optimizes identity embeddings via gradient descent, while
regularizing embeddings on the face manifold.
However, HyperFace maximizes inter-class separation \emph{among
synthetics only}; non-collision with real humans is not a
design objective.
BIP differs fundamentally: the enrollment gallery $\mathcal{R}$ is a
hard constraint, non-collision with every identity in $\mathcal{R}$
at threshold $\tau$ is the primary requirement, and our geometric
analysis shows that valid virtual identities must reside within the
real identity manifold rather than a geometrically separate space.

\Paragraph{Identity-Aware Forensics.}
Deepfake and face forgery detection has  been formulated
as binary classification between real and manipulated images~\cite{dang2020detection}.
Many detectors exploit low-level artifacts, frequency-domain cues,
blending boundaries, or generator-specific
traces~\cite{rossler2019faceforensics++,shiohara2022detecting,
gu2022region,qian2020thinking}, but they often generalize poorly across manipulation types and generation models, and benchmark accuracy may overstate forensic ability~\cite{pagon2026deepfake}.
Identity-aware forensics treats identity consistency as a forensic
signal: ID-Reveal~\cite{cozzolino2021id} learns person-specific motion
patterns and detects manipulations as deviations from expected behavior,
while other works exploit identity inconsistency or face recognition
features to detect face swaps~\cite{xu2024identity,kim2025selfi}.
These methods assume that real identities already exist and ask
whether a given image is manipulated.
BIP addresses a complementary question: how to provision new biometric
identities guaranteed to lie outside all enrolled real identity
territories. The resulting non-overlapping partition of the embedding
space between real and virtual identities is a direct benefit of
the BIP constraints, not a product of forensic training.

\begin{figure}
  \centering
  \includegraphics[width=\linewidth]{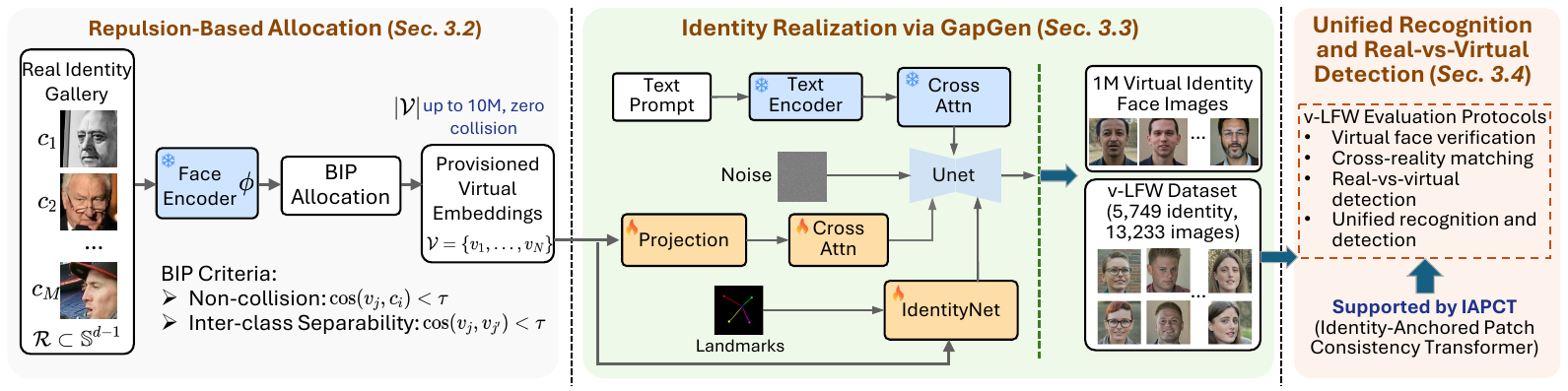}
\vspace{-5mm}
\caption{\textbf{BIP pipeline}. Left: Repulsion-based allocation provisions
$\mathcal{V}{=}\{v_1,\ldots,v_N\}$ satisfying
$\cos(v_j,c_i){<}\tau$ and $\cos(v_j,v_{j'}){<}\tau$,
scaling to $|\mathcal{V}|{=}10$M with zero observed collision. Middle: GapGen renders each $s \in \mathcal{V}$
into a $1024{\times}1024$ face image $\tilde{x}{=}G(s)$,
producing 1M virtual identity images and the v-LFW
benchmark.
Right: v-LFW supports four protocols spanning virtual
face verification, cross-reality matching, real-vs-virtual
detection, and unified recognition and detection, supported by
IAPCT as a lightweight diagnostic tool.}
\label{fig:overview}
\end{figure}

\section{Methodology}
\label{sec:method}

\subsection{Formalization}
\label{sec:formalization}

Let $\phi: \mathcal{X} \rightarrow \mathbb{S}^{d-1}$ denote a face
recognition encoder that maps a face image $x$ to a $d$-dimensional
L2-normalized embedding on the unit hypersphere $\mathbb{S}^{d-1}$.

\begin{assumption}[Angular-margin Encoder]
\label{ass:encoder}
The encoder $\phi$ is trained with an angular margin loss
(\emph{e.g.}, ArcFace~\cite{deng2019arcface} or
AdaFace~\cite{kim2022adaface}).
Embeddings with $\cos(\phi(x), \phi(x')) \geq \tau$ are accepted as
the same identity, where $\tau \in (0,1)$ is set at the operating
point of the recognition system (see Appx.~\ref{app:threshold} for details on $\tau$).
Equivalently, $\tau$ defines an angular
boundary on $\mathbb{S}^{d-1}$: any two embeddings within angular
distance $\arccos(\tau)$ are recognized as the same identity.
\end{assumption}

Each real identity $i \in \{1,\ldots,M\}$ with images
$\mathcal{X}_i = \{x_i^{(1)}, \ldots, x_i^{(n_i)}\}$ is represented
by the dominant direction of its embedding cluster:
\begin{equation}
  c_i = \mathrm{normalize}\!\left(
        \sum_{k=1}^{n_i} \phi\!\left(x_i^{(k)}\right)
        \right),
  \label{eq:centroid}
\end{equation}
with $c_i = \phi(x_i^{(1)})$ when $n_i = 1$.
The enrollment gallery is
$\mathcal{R} = \{c_1, \ldots, c_M\} \subset \mathbb{S}^{d-1}$.

\begin{definition}[Biometric Identity Provisioning]
\label{def:bip}
Given an encoder $\phi$ satisfying Assumption~\ref{ass:encoder},
an enrollment gallery $\mathcal{R} \subset \mathbb{S}^{d-1}$ of $M$
real identity centroids, a verification threshold $\tau$, and a
provisioning count $N$, the \textbf{Biometric Identity Provisioning
(BIP)} problem is to generate
$\mathcal{V} = \{v_1, \ldots, v_N\} \subset \mathbb{S}^{d-1}$
satisfying:
\begin{align}
  \cos(v_j, c_i) &< \tau,
  \quad \forall i \in \{1,\ldots,M\},\; \forall j \in \{1,\ldots,N\}
  \label{eq:noncollision} \\
  \cos(v_j, v_{j'}) &< \tau,
  \quad \forall j \neq j'.
  \label{eq:separability}
\end{align}
\end{definition}
The realizability requirement, that each $v_j$ corresponds to a
high-fidelity face image whose re-encoded embedding satisfies
Eqs.~\eqref{eq:noncollision}--\eqref{eq:separability}, is
addressed by the gap-aware generator in Sec.~\ref{sec:realization}.
Fig.~\ref{fig:overview} summarizes the full pipeline of allocation, realization, and evaluation.

\begin{remark}
\label{rem:galleryrelative}
The constraints in Definition~\ref{def:bip} are
\emph{encoder-specific} and \emph{gallery-relative}: they are stated
for a given $\phi$ and $\tau$, and are guaranteed
w.r.t.~$\mathcal{R}$ at provisioning time.
BIP identities are therefore long-lived and lifecycle-managed rather
than permanent.
When new real identities $\Delta\mathcal{R}$ are subsequently enrolled,
any $v_j$ with $\cos(v_j, r') \geq \tau$ for some
$r' \in \Delta\mathcal{R}$ must be revoked and reassigned against the
updated gallery $\mathcal{R} \cup \Delta\mathcal{R}$; the enrolled
real identity is never modified, and the rest of $\mathcal{V}$ is
never re-provisioned.
This requires $O(|\mathcal{V}| \times |\Delta\mathcal{R}|)$ cosine
operations in the exact case, reducible with approximate
nearest-neighbor indexing~\cite{johnson2021faiss}.
Provisioning with a safety buffer $\tau_\text{safe} = \tau - \Delta$
enlarges the margin to the enrolled gallery and adds a monitoring
zone $[\tau_\text{safe}, \tau)$ (Proposition~\ref{prop:buffer} and
the lifecycle policy in Appx.~\ref{app:safety_buffer}); open-world
behavior is studied in Sec.~\ref{sec:exp}.
\end{remark}

\Paragraph{BIP Criteria.}
Three metrics directly operationalize Definition~\ref{def:bip}.
\textbf{\emph{(i) Non-collision:}} the percentage of provisioned
virtual identities $v_j \in \mathcal{V}$ satisfying
Eq.~\eqref{eq:noncollision} against all $M$ enrolled real centroids
in $\mathcal{R}$; a score below 100\% means some virtual identities
are indistinguishable from real humans in the recognition
system, posing a direct privacy risk.
\textbf{\emph{(ii) Inter-class separability (Inter-Sep):}} the
percentage of virtual identities $v_j$ satisfying
Eq.~\eqref{eq:separability} against every other identity in
$\mathcal{V}$; a score below 100\% means some virtual
entities share an identity, making them unreliable as distinct
assignable identities.
\textbf{\emph{(iii) FID}}~\cite{heusel2017gans} (against
FFHQ~\cite{karras2019style}) measures perceptual image quality of
generated virtual identities against portrait photographs.
Prior synthetic face methods produce low-resolution
$112{\times}112$ training crops, whereas digital entities require
portrait-quality images displayable and recognizable in physical
environments.

\begin{figure}
  \centering
  \includegraphics[width=\linewidth]{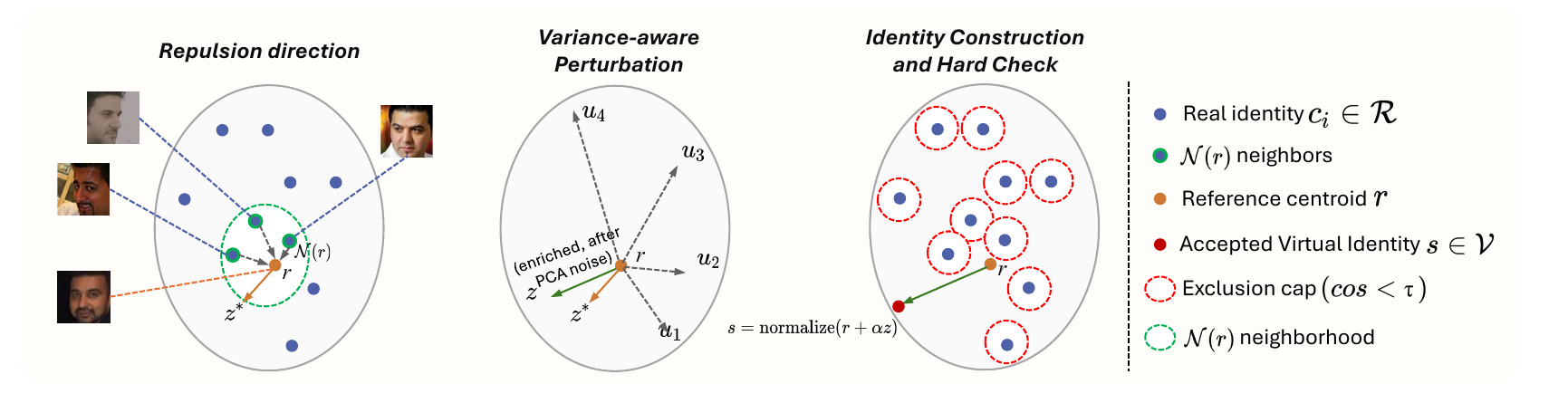}
  \caption{ \textbf{Repulsion-based virtual identity allocation.} Left: $z^*{=}{-}m/\|m\|_2$ points away from the
weighted centroid of nearest real neighbors $\mathcal{N}(r)$.
Middle: $z^*$ is enriched with PCA-aware noise
along $\{u_k\}$ (weighted by $\sigma_k$) to obtain $z$.
Right: Candidate $s{=}\mathrm{normalize}(r{+}\alpha z)$
is accepted only if $\cos(s,c_i){<}\tau$ for all
$c_i\in\mathcal{R}$ and $\cos(s,v_j){<}\tau$ for all
$v_j\in\mathcal{V}_t$; failed candidates are resampled.}
  \label{fig:allocation}
\end{figure}

\subsection{Geometry of the Real Identity Manifold and Repulsion-Based Allocation}
\label{sec:geometry}

\Paragraph{Geometry of the Real Identity Manifold.}
ArcFace ResNet-100~\cite{deng2019arcface} trained on
Glint360K~\cite{an2021partial} is used as $\phi$ ($d{=}512$), and all
$M{=}360{,}232$ identity centroids form $\mathcal{R}$.
To characterize where virtual identities can be validly placed and
to guide provisioning diversity, we perform PCA on these centroids
in the ambient Euclidean space, yielding eigenvectors
$U = [u_1, \ldots, u_d]$ with eigenvalues
$\lambda_1 \geq \cdots \geq \lambda_d \geq 0$, where
$k \in \{1,\ldots,d\}$ denotes rank in explained variance, and
per-direction standard deviations $\sigma_k = \sqrt{\lambda_k}$;
both $\{u_k\}$ and $\{\sigma_k\}$ are reused in the allocation
strategy below.

\begin{observation}[Low-dimensional Real Identity Manifold]
\label{prop:manifold}
Let the \emph{principal energy} for dimension $k$ be
$E(k) = \bigl(\sum_{k'=1}^{k}\lambda_{k'}\bigr) /
\bigl(\sum_{k'=1}^{d}\lambda_{k'}\bigr)$,
where $k' \in \{1,\ldots,d\}$ is a summation index.
The real identity centroids in $\mathcal{R}$ concentrate over 95\%
of their total variance within the top $k{=}269$ principal components
(Fig.~\ref{fig:pca_appendix} in Appx.~\ref{app:pca}), despite residing in $\mathbb{S}^{511}$.
%
Since $\phi$ is trained on real face images, realizable face
embeddings are expected to remain concentrated near the dominant
real-face subspace.
Low-variance residual directions may exist in the ambient
embedding space, but they are unlikely to provide a reliable
region for high-fidelity face realization.
Valid virtual identities should therefore occupy
\emph{unclaimed gaps within the real face manifold},
rather than a geometrically separate region.
\end{observation}
Observation~\ref{prop:manifold} rules out orthogonal or
free-subspace constructions~\cite{kim2025vigface} and motivates
the gap-allocation strategy below.

\Paragraph{Capacity.}
Modeling the real face manifold as $\mathbb{S}^{\tilde{d}-1}$ with
$\tilde{d}{=}269$ (Observation~\ref{prop:manifold}), the
Gilbert--Varshamov bound gives
$A_{\mathrm{pack}}(\tilde{d},\tau) \geq A_{\mathrm{GV}}(\tilde{d},\tau)
= 1/\mu(\tau,\tilde{d})$, where $A_{\mathrm{pack}}$ is the largest
number of points on this sphere with pairwise cosine below $\tau$ and
$\mu(\tau,\tilde{d})$ is the normalized volume of a spherical cap of
cosine radius $\tau$
(Proposition~\ref{prop:capacity}, Appx.~\ref{app:capacity}).
At $\tau{=}0.391$, $A_{\mathrm{GV}} = 7.41\times10^{10} \approx 2^{36.1}$,
about $7{,}400$ times the 10M identities provisioned in
Sec.~\ref{sec:exp}.
This is a model-based headroom for allocation in embedding space;
separability after realization is measured in Sec.~\ref{sec:exp}.

\Paragraph{Repulsion Direction.}
Given a reference centroid $r \in \mathcal{R}$, let
$\mathcal{N}(r) = \{c_{n_1}, \ldots, c_{n_K}\}$ be its $K$ nearest
neighbors in $\mathcal{R}$, with softmax weights:
\begin{equation}
  w_{n_k} = \frac{\exp(-d(r,\, c_{n_k})/t)}
             {\sum_{k'=1}^{K} \exp(-d(r,\, c_{n_{k'}})/t)},
\end{equation}
where $d(r, c) = 1 - \cos(r, c)$ and $t$ denotes temperature.

\begin{remark}[Repulsion Direction]
\label{rem:repulsion}
Define $z^*$ as the negation of the weighted neighborhood centroid:
\begin{equation}
  z^* = -\,\frac{\sum_{k=1}^{K} w_{n_k}\, c_{n_k}}
               {\bigl\|\sum_{k=1}^{K} w_{n_k}\, c_{n_k}\bigr\|_2}.
  \label{eq:zstar}
\end{equation}
$z^*$ serves as a \emph{repulsive heuristic}: it moves
candidates away from the weighted neighborhood centroid of $r$,
increasing the probability that a candidate satisfies the
constraints of Definition~\ref{def:bip}.
It does not guarantee them; non-collision is enforced by the exact
checks applied to every candidate before acceptance
(see Appx.~\ref{app:repulsion} for an analysis of $z^*$
after spherical normalization).
\end{remark}

To improve diversity while keeping provisioned identities on the
real face manifold, we enrich $z^*$ with \emph{PCA-aware noise},
random perturbations aligned to the principal directions $\{u_k\}$:
\begin{equation}
  z = \mathrm{normalize}\!\left(
      z^* + \kappa \sum_{k=1}^{d} \eta_k \sigma_k u_k\right),
  \quad \eta_k \sim \mathcal{N}(0,1),
  \label{eq:perturbation}
\end{equation}
where $\kappa \geq 0$ balances direction fidelity against sample
diversity.
Weighting by $\sigma_k$ ensures perturbations are larger along
high-variance principal directions and smaller along low-variance
ones, keeping provisioned embeddings aligned with the natural
directions of variation in $\mathcal{R}$.
The relationship between $\kappa$, $\alpha$, and the
theoretical face manifold capacity is analyzed in
Appx.~\ref{app:capacity} (Sec.~\ref{app:alpha_kappa}).

\Paragraph{Identity Construction and Hard Check.}
A candidate virtual identity is constructed as:
\begin{equation}
  s = \mathrm{normalize}(r + \alpha z),
  \label{eq:candidate}
\end{equation}
where $r \in \mathcal{R}$ is the reference centroid,
$z$ is from Eq.~\eqref{eq:perturbation}, and $\alpha > 0$ controls
perturbation strength.

\begin{lemma}[Effect of Perturbation Strength]
\label{lem:displacement}
For any $r, z \in \mathbb{S}^{d-1}$, the cosine similarity between
$s = \mathrm{normalize}(r + \alpha z)$ and $r$ satisfies:
\begin{equation}
  \cos(s,\, r) = \frac{1 + \alpha\,(r \cdot z)}
                      {\sqrt{1 + 2\alpha\,(r \cdot z) + \alpha^2}}.
  \label{eq:cosgeneral}
\end{equation}
In the special case $\kappa{=}0$, $z = z^*$; as $r$ is positively
correlated with its nearest neighbors $\mathcal{N}(r)$ on
$\mathbb{S}^{d-1}$, the weighted centroid $\sum_k w_{n_k} c_{n_k}$
is positively aligned with $r$, giving $r \cdot z^* < 0$, and
$\cos(s, r)$ is strictly decreasing in $\alpha$.
In the orthogonal case $r \perp z$, Eq.~\eqref{eq:cosgeneral}
simplifies to $\cos(s, r) = 1/\sqrt{1+\alpha^2}$.
\end{lemma}

The parameter $\alpha$ controls the fundamental trade-off in BIP, as quantified in Tab.~\ref{tab:alpha_threshold_main}:
larger $\alpha$ increases angular displacement from real identity
territories improving non-collision, while smaller $\alpha$ keeps
provisioned identities closer to the face manifold improving
realizability. 
Each candidate $s$ is accepted only when it satisfies the BIP
constraints of Definition~\ref{def:bip} with respect to the
current state of the gallery:
\begin{align}
  \cos(s,\, c_i) &< \tau
  \quad \forall\, i \in \{1,\ldots,M\}
  \quad \text{(non-collision with } \mathcal{R}\text{)},
  \label{eq:hardcheck1}\\
  \cos(s,\, v_j) &< \tau
  \quad \forall\, v_j \in \mathcal{V}_t
  \quad \text{(separability within } \mathcal{V}_t\text{)},
  \label{eq:hardcheck2}
\end{align}
where $\mathcal{V}_t$ denotes the set of accepted virtual identities
at step $t$; failed candidates are resampled with fresh $\eta$.
%
By induction, if every candidate passes these checks before
acceptance, the final $\mathcal{V}$ satisfies
Eqs.~\eqref{eq:noncollision}--\eqref{eq:separability} by
construction.
Appx.~\ref{app:capacity} (Corollary~\ref{cor:alpha}) establishes
that $\alpha > \sqrt{1-\tau^2}/\tau$ is sufficient for non-collision
in the repulsive case; the full $\alpha$ ablation is in
Tab.~\ref{tab:alpha_threshold_main}.
We demonstrate 10M
non-colliding virtual identity embeddings against
$|\mathcal{R}|{=}360$K with no observed collision at $\tau=0.391$,
well within the capacity headroom above.
Fig.~\ref{fig:allocation} shows the three components: repulsion direction, PCA-aware perturbation, and hard non-collision check.

\subsection{Identity Realization via GapGen}
\label{sec:realization}

A provisioned embedding $s \in \mathcal{V}$ defines a target location
in the biometric space; \textbf{GapGen} realizes it as a face image
$\tilde{x}$ such that $\phi(\tilde{x})$ remains faithful to $s$.
The challenge is that, since $s$ is intentionally displaced from the real
identity distribution, a pretrained generator $G$, optimized on
real embeddings, exhibits reduced fidelity when conditioned on such
out-of-distribution inputs.

\Paragraph{Identity-conditioned Generator.}
We build on InstantID~\cite{wang2024instantid} as the
identity-conditioned generator $G$, which accepts an ArcFace
embedding as its condition and generates $1024{\times}1024$ face
images, denoted as $\tilde{x} = G(s)$.
Since BIP uses the same ArcFace embedding space for both allocation
and generation, provisioned embeddings $s \in \mathcal{V}$ can be
passed directly as conditions without additional projection. Full implementation details are in Appx.~\ref{app:gapgen}.

\Paragraph{GapGen Fine-tuning.}
We fine-tune $G$ with a curriculum that interleaves \emph{real steps}
and \emph{virtual steps} at mixing ratio $\beta \in [0,1]$.
Both step types share the round-trip identity loss:
\begin{equation}
  \mathcal{L}_{\mathrm{RT}}(\tilde{x},\, e^*) =
  1 - \cos\!\left(\phi(\tilde{x}),\, e^*\right),
  \label{eq:lrt}
\end{equation}
where $e^*$ is the target identity embedding and
$\rho_{\mathrm{RT}} = \cos(\phi(\tilde{x}), e^*)$ is the
\emph{round-trip similarity} quantifying rendering fidelity.

\emph{Real steps} train $G$ on real images $x$ from
FFHQ~\cite{karras2019style}, chosen for its $1024{\times}1024$
resolution consistent with our generation target, with target
$e^* = \phi(x)$:
\begin{equation}
  \mathcal{L}_{\mathrm{real}} =
  \mathcal{L}_{\mathrm{denoise}} +
  \lambda_{\mathrm{id}}\,\mathcal{L}_{\mathrm{RT}}(\tilde{x},\,\phi(x)) +
  \lambda_{\mathrm{perc}}\,\mathcal{L}_{\mathrm{perc}},
  \label{eq:lreal}
\end{equation}
where $\mathcal{L}_{\mathrm{denoise}} =
\mathbb{E}_{t,\epsilon}[\|\epsilon - \epsilon_\theta(x_t,t,\phi(x))\|^2]$
is the standard noise prediction loss~\cite{ho2020ddpm}, and
$\mathcal{L}_{\mathrm{perc}}$ is a perceptual similarity
loss~\cite{zhang2018unreasonable}.
Real steps adapt $G$ to the high-resolution portrait domain and
preserve identity fidelity on in-distribution embeddings.

\emph{Virtual steps} address the core challenge of BIP realization:
provisioned embeddings $s \in \mathcal{V}$ have no paired
ground-truth images, as virtual identities have no physical
counterpart.
At each virtual step, we sample a reference centroid $r \in \mathcal{R}$,
draw $\alpha \sim \mathrm{Uniform}(\alpha_{\min}{=}2, \alpha_{\max}{=}5)$,
construct $s = \mathrm{normalize}(r + \alpha z)$ on-the-fly,
generate $\tilde{x} = G(s)$ via truncated DDIM sampling~\cite{song2020denoising}, 
and optimize $\mathcal{L}_\text{RT}(\tilde{x}, s)$
(Eq.~\eqref{eq:lrt} with $e^*{=}s$).
%
This unsupervised objective drives $G$ to extend its generative
capability into the non-collision region of the embedding space,
\emph{without using paired training data}.


\Paragraph{Intra-class Variation.}
Multiple images per virtual identity are generated by varying facial
pose, lighting conditions, photographic style, and background through
different pose configurations and text style prompts, producing
within-identity variation comparable to unconstrained portrait
photography.

\subsection{Unified Recognition and Real-vs-Virtual Detection}
\label{sec:iapct}
The BIP criteria of Sec.~\ref{sec:formalization} evaluate
virtual identities: whether they are
non-colliding, separable, and realizable.
Deploying digital entities alongside real humans raises a
complementary set of questions that no existing benchmark addresses:
can virtual identities be verified against themselves?
How do they interact with real identities under standard recognition
protocols?
Can a deployed system reliably distinguish real from virtual?
\textbf{v-LFW} is built to answer these questions.

\Paragraph{Benchmark Construction.}
We construct \textbf{v-LFW} to mirror LFW~\cite{huang2007lfw}
in both identity count (5{,}749) and image count (13{,}233),
supporting controlled evaluation under established face recognition
protocols.
Each virtual identity is provisioned against $\mathcal{R}$ augmented
with all LFW identities, ensuring non-collision with the benchmark.
Images are rendered at $1024{\times}1024$ with diverse pose and
style variation, then filtered for face quality.
We define four protocols: virtual face verification, cross-reality
matching, real-vs-virtual detection, and unified recognition/detection.

\Paragraph{IAPCT: A Diagnostic Evaluation Tool.}
To support the real-vs-virtual detection and unified recognition
protocols of v-LFW, we provide \textbf{IAPCT} (Identity-Anchored
Patch Consistency Transformer) as a lightweight diagnostic tool
built on the frozen encoder $\phi$; it is not a primary
contribution of this work.
Identity-conditioned generators match the target embedding at the
output level but impose no constraint on intermediate feature maps,
creating detectable inconsistencies between local patch statistics
and the global identity embedding $e = \phi(x)$ that IAPCT exploits.
Let $e = \phi(x) \in \mathbb{S}^{d-1}$ denote the identity
embedding from the frozen backbone.
Multi-scale spatial tokens $\{t_{l,p}\}$ are extracted from four
intermediate feature maps $\{F_l\}_{l=1}^{4}$ and projected to
$d_{\mathrm{model}}{=}256$ via per-stage linear layers.
The projected identity anchor $e_{\mathrm{proj}} = W_{\mathrm{id}}\,e$
serves as the cross-attention key in a 4-layer transformer
encoder~\cite{vaswani2017attention,dosovitskiy2021vit}; each layer applies
self-attention among spatial tokens followed by cross-attention
to $e_{\mathrm{proj}}$, producing patch-identity consistency scores:
\begin{equation}
  \gamma_{l,p} = \sigma\!\left(
  \frac{(W_Q\,t_{l,p}) \cdot (W_K\,e_{\mathrm{proj}})}{\sqrt{d_k}}
  \right) \in (0,1),
  \label{eq:gamma}
\end{equation}
where $\sigma$ is the sigmoid function and $\gamma_{l,p}$ measures
the consistency of patch $p$ at stage $l$ with the claimed identity
$e$.
Real faces yield uniformly high $\gamma_{l,p}$ since every spatial
region originates from the same physical person; virtual identity
images yield heterogeneous $\gamma_{l,p}$ due to unconstrained
intermediate features.
A \texttt{[CLS]} token produces $\hat{y}$ via an MLP head with
loss $\mathcal{L} {=} \mathrm{BCE}(\hat{y},y) {+}
\lambda_c(H(\gamma^\text{virtual}) {-} H(\gamma^\text{real}))$,
where $H(\gamma^{(x)}) {=} {-}\sum_{l,p}\bar{\gamma}_{l,p}^{(x)}
\log\bar{\gamma}_{l,p}^{(x)}$ is the entropy of the normalized
attention distribution and $\phi$ remains frozen.
The identity embedding $e{=}\phi(x)$ simultaneously serves face
verification at no additional backbone cost.
Full details are in Appx.~\ref{app:iapct}.

\section{Experiments}

\label{sec:exp}

\Paragraph{Baselines.}
We compare BIP with representative synthetic face generation
methods:
\emph{DCFace}~\cite{kim2023dcface}, a dual-condition diffusion
model optimized for inter-class separability among synthetics;
\emph{Arc2Face}~\cite{papantoniou2024arc2face}, which conditions
generation on real ArcFace embeddings;
and \emph{Vec2Face}~\cite{wu2024vec2face}, which synthesizes
identity-consistent faces from target embeddings with explicit
separability control.
All baselines are designed for generating face recognition training data,
not biometric identity provisioning.

\Paragraph{BIP Configuration.}
We use ArcFace ResNet-100~\cite{deng2019arcface} trained on
Glint360K~\cite{an2021partial} as $\phi$ ($d{=}512$,
$M{=}360{,}232$).
Unless otherwise stated, we set $\tau{=}0.391$
(FAR$\approx2\times10^{-5}$ on IJB-B; Appx.~\ref{app:threshold}),
$\alpha{=}4.0$, $K{=}10$ nearest neighbors, temperature $t{=}0.1$,
and PCA noise weight $\kappa{=}1.0$.
GapGen is fine-tuned from InstantID~\cite{wang2024instantid}
with mixing ratio $\beta{=}0.2$, $\lambda_{\mathrm{id}}{=}0.1$, $\lambda_{\mathrm{perc}}{=}0.1$, generating faces of
$1024{\times}1024$ pixels.
The large-scale setting provisions \textbf{\emph{10 million}} virtual identities; image realization is validated at \textbf{\emph{1 million}} identities.
Additional details are in the corresponding sections of the Appx.

\begin{table}[t]
\centering
\small
\setlength{\tabcolsep}{5pt}
\renewcommand{\arraystretch}{0.98}
\caption{\textbf{BIP Non-Collision / Inter-Class Sep (\%) across
identity scale $|\mathcal{V}|$, perturbation strength $\alpha$,
and threshold $\tau$.}
At the operating point $\tau{=}0.391$, $\alpha{\geq}4$
achieves 100\% Non-Collision at every scale, with Inter-Sep of
100\% up to 1M and 99.98\% at 10M.}
\label{tab:alpha_threshold_main}
\resizebox{0.85\textwidth}{!}{%
\begin{tabular}{c|c|cccccc}
\toprule
$|\mathcal{V}|$ & $\alpha$ & $\tau=0.448$  & $\tau=\textbf{0.391}$  & $\tau=0.360$   &  $\tau=0.341$  &  $\tau=0.330$  & $\tau=0.319$   \\
\midrule
\multirow{4}{*}{100K}
& 2 & $100.00 / 100.00$ & $85.63 / 100.00$  & $64.57 / 99.99$ & $48.69 / 99.93$ & $39.05 / 99.81$ & $30.55 / 99.54$ \\
& 3 & $100.00 / 100.00$ & $99.98 / 100.00$  & $99.80 / 99.99$ & $98.96 / 99.93$ & $97.46 / 99.81$ & $94.28 / 99.55$ \\
& 4 & $100.00 / 100.00$ & $\textbf{100.00 / 100.00}$ & $99.91 / 99.99$ & $99.39 / 99.94$ & $98.19 / 99.82$ & $95.55 / 99.54$ \\
& 5 & $100.00 / 100.00$ & $100.00 / 100.00$ & $99.91 / 99.99$ & $99.40 / 99.94$ & $98.23 / 99.81$ & $95.62 / 99.53$ \\
\midrule
\multirow{4}{*}{1M}
& 2 & $100.00 / 100.00$ & $85.61 / 100.00$  & $64.48 / 99.92$ & $48.51 / 99.42$ & $39.01 / 98.24$ & $30.42 / 95.48$ \\
& 3 & $100.00 / 100.00$ & $99.98 / 100.00$  & $99.78 / 99.93$ & $98.94 / 99.43$ & $97.42 / 98.28$ & $94.20 / 95.54$ \\
& 4 & $100.00 / 100.00$ & $\textbf{100.00 / 100.00}$ & $99.91 / 99.92$ & $99.38 / 99.43$ & $98.17 / 98.27$ & $95.52 / 95.53$ \\
& 5 & $100.00 / 100.00$ & $100.00 / 100.00$ & $99.91 / 99.92$ & $99.38 / 99.43$ & $98.22 / 98.27$ & $95.53 / 95.54$ \\
\midrule
\multirow{4}{*}{10M}
& 2 & $100.00 / 100.00$ & $85.10 / 99.97$  & $64.74 / 99.22$ & $48.06 / 94.51$ & $39.00 / 85.55$ & $30.16 / 66.62$ \\
& 3 & $100.00 / 100.00$ & $99.98 / 99.98$  & $99.77 / 99.21$ & $98.99 / 94.67$ & $97.39 / 85.50$ & $94.16 / 67.06$ \\
& 4 & $100.00 / 100.00$ & $\textbf{100.00 / 99.98}$ & $99.91 / 99.24$ & $99.39 / 94.66$ & $98.27 / 85.57$ & $95.36 / 66.99$ \\
& 5 & $100.00 / 100.00$ & $100.00 / 99.98$ & $99.91 / 99.23$ & $99.40 / 94.64$ & $98.28 / 85.53$ & $95.38 / 66.94$ \\

\bottomrule
\end{tabular}}
\end{table}

\begin{table}[t]
\centering
\small
\setlength{\tabcolsep}{4pt}
\renewcommand{\arraystretch}{0.9}
\caption{\textbf{Synthetic face generation under BIP criteria.}
All methods are re-encoded by the same frozen ArcFace, tested
against $\mathcal{R}$ ($M{=}360$K) at $\tau{=}0.391$.
FID measures perceptual image quality against
FFHQ. RT pass rate reports the fraction of generated images
whose re-encoded embedding satisfies $\rho_\text{RT}{>}0.6$,
quantifying how faithfully the generator realizes the target
identity embedding. CASIA-WebFace is shown as a real-data reference. Vec2Face* renders identities allocated by BIP in its own embedding space.
}
\label{tab:main}
\resizebox{0.90\linewidth}{!}{%
\begin{tabular}{lcccccc}
\toprule
Method & \#Identity & Non-Collision~$\uparrow$ & Inter-Sep~$\uparrow$ &
FID~$\downarrow$ & RT pass rate~$\uparrow$ & Resolution \\
\midrule
CAS.\ (Real)~\cite{yi2014learning} & 10.6K
& $0.32$ & $90.08$ & $131.70$ &   - & $250{\times}250$ \\
\midrule
DCFace~\cite{kim2023dcface} & 10K
& $72.38$ & $49.20$ & $139.70$  &    - & $112{\times}112$ \\
Arc2Face~\cite{papantoniou2024arc2face} & 10K
& $93.90$ & $88.67$ & $74.11$  & $0.00$  & $512{\times}512$ \\
Vec2Face~\cite{wu2024vec2face} & 10K
& $53.92$  & $65.02$ & $159.30$  &  $62.78$ & $112{\times}112$ \\
Vec2Face~\cite{wu2024vec2face} & 100K
& $52.95$ & $45.44$ & $159.18$   &  $61.43  $ & $112{\times}112$ \\
Vec2Face~\cite{wu2024vec2face} & 500K
& $52.93$ & $36.22$ & $159.29$   & $62.01$  & $112{\times}112$ \\

 Vec2Face*~\cite{wu2024vec2face} & 500K
& $98.45$ & $75.61$ & -   & - & $112{\times}112$ \\
\midrule

\textbf{BIP+GapGen (ours)} & 10K
& $\mathbf{98.91}$ & $\mathbf{99.57}$ & $56.20$   &  $87.94$  & $1024{\times}1024$ \\
\textbf{BIP+GapGen (ours)} & 100K
& $98.67$ & $97.84$ & $\mathbf{54.63}$   &  $\textbf{89.42}$ & $1024{\times}1024$ \\
\textbf{BIP+GapGen (ours)} & 500K
& $98.62$ & $86.30$ & $55.62$ &  $88.31$ & $1024{\times}1024$ \\
\textbf{BIP+GapGen (ours)} & 1M
& $98.61$ & $78.36$ & $56.47$  &  $88.75$ & $1024{\times}1024$ \\
\bottomrule
\end{tabular}}
\end{table}

\begin{figure}[t]
  \centering
  \includegraphics[width=\linewidth]{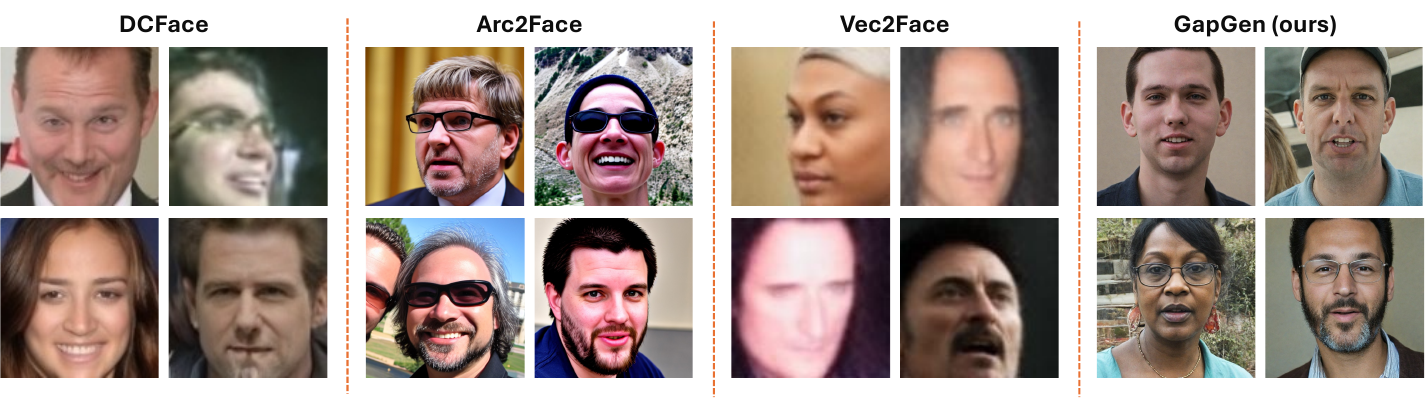}
\caption{Qualitative comparison with synthetic face generators.
Four randomly sampled identities per method.
GapGen produces photorealistic $1024{\times}1024$ portrait faces
with natural texture and coherent lighting, free from blurring,
distortion, and uncontrolled backgrounds of the baselines.}
\label{fig:qualitative}
\end{figure}


\Paragraph{Geometry and Scale: Non-Collision Holds at Ten Million.}
%
%
%
Tab.~\ref{tab:alpha_threshold_main} jointly ablates perturbation
strength $\alpha$, provisioning scale $|\mathcal{V}|$, and
verification threshold $\tau$ across the full IJB-B operating range
(see Fig.~\ref{fig:roc} in Appx.~\ref{app:threshold}).
At the primary threshold $\tau{=}0.391$ and $\alpha{\geq}4$,
Non-Collision reaches \textbf{100\% at 100K, 1M, and 10M scale},
with Inter-Sep at 100\% up to 1M and 99.98\% at 10M: ten million
virtual identities are provisioned against 360K real identities with
zero observed collision, consistent with the theoretical sufficient
bound $\alpha^*(0,0.391){\approx}2.35$ (Corollary~\ref{cor:alpha}).
$\alpha{=}2$ falls below this threshold and yields only $85.10$--$85.63\%$
Non-Collision, while $\alpha{\geq}3$ already recovers near-perfect
performance, confirming the theoretical prediction.
Scaling from 100K to 10M incurs negligible cost: Non-Collision
remains at 100\% throughout, and Inter-Sep stays at or above $99.98\%$
at $\tau{=}0.391$, demonstrating that BIP allocation is not
capacity-limited at any scale tested.
At stricter thresholds (smaller $\tau$), Inter-Sep degrades more
noticeably at 10M (\emph{e.g.}, $66.99\%$ at $\tau{=}0.319$), consistent
with the tighter capacity headroom at extreme operating points
(Appx.~\ref{app:capacity}); $\tau{=}0.391$ remains robust
across all scales.
All sets in Tab.~\ref{tab:alpha_threshold_main} are provisioned with the hard checks of Eqs.~\eqref{eq:hardcheck1}--\eqref{eq:hardcheck2} at $\tau_{\mathrm{check}}{=}0.435$ and evaluated at each $\tau$; for $\tau{<}\tau_{\mathrm{check}}$, the scores thus reflect how well $\alpha$ alone keeps candidates clear of $\mathcal{R}$.

\Paragraph{Realization Quality: GapGen Outperforms Baselines across All BIP Criteria.}
Tab.~\ref{tab:main} compares BIP+GapGen against baselines under the
three BIP criteria at $\tau{=}0.391$.
While BIP allocation scales to 10M embeddings, image realization
is validated up to 1M virtual images due to the computational cost
of $1024{\times}1024$ generation; this reflects a
deliberate fidelity-scale trade-off rather than an allocation
bottleneck.
Existing methods fail on at least one criterion:
DCFace~\cite{kim2023dcface} reaches only $49.20\%$ Inter-Sep;
Vec2Face~\cite{wu2024vec2face} (reproduced from released code)
achieves only $53.92\%$ Non-Collision at 10K with Inter-Sep
collapsing to $36.22\%$ at 500K;
Arc2Face~\cite{papantoniou2024arc2face} achieves $93.90\%$
Non-Collision but with an RT pass rate of $0.00$, indicating
complete failure of embedding fidelity after re-encoding.
BIP+GapGen achieves $98.91\%$ Non-Collision, $99.57\%$
Inter-Sep, FID $56.20$, and round-trip pass rate $87.94\%$
at 10K, the best across all criteria at up to $9{\times}$ higher
resolution than baselines.
As scale increases, Non-Collision remains stable ($98.61\%$
at 1M); Inter-Sep decreases to $78.36\%$, reflecting the
realization gap between embedding-space allocation
(Tab.~\ref{tab:alpha_threshold_main}, Inter-Sep $\geq99.98\%$)
and image-level realization at scale.
Allocation capacity and realized capacity therefore differ.
Under the same encoder, threshold, and protocol, re-encoded Inter-Sep
stays above $99\%$ up to $29{,}830$ identities and above $90\%$ up to
$300{,}132$, against $12{,}777$ and $141{,}248$ for real Glint360K
identities (Appx.~\ref{app:realized_capacity}).
Fig.~\ref{fig:qualitative} confirms the qualitative advantage;
additional visual results are in Appx.~\ref{app:vlfw}
and~\ref{app:add-grid}.

\Paragraph{Deepfake Detectability of Virtual Faces.}
\label{sec:safety}
A concern with high-fidelity virtual face generation is whether
provisioned identities could be misused as deepfakes or synthetic
personas.
We evaluate this by applying five SoTA deepfake detectors to our
generated virtual faces zero-shot, without fine-tuning on
virtual faces; all detectors are trained on
AIFaceFairnessBench~\cite{lin2025aiface}.
As in Tab.~\ref{tab:safety}, all detectors achieve
high accuracy on BIP-generated faces, indicating that existing
detectors already flag provisioned virtual identities
without modification.
The non-collision guarantee and deepfake detectability are
complementary: the former ensures virtual identities are
biometrically distinct from all real humans, while the
latter confirms they remain within the reach of existing forensic
systems.

\begin{wrapfigure}[11]{r}{0.38\linewidth}
\vspace{-10pt}
\centering
\includegraphics[width=\linewidth]{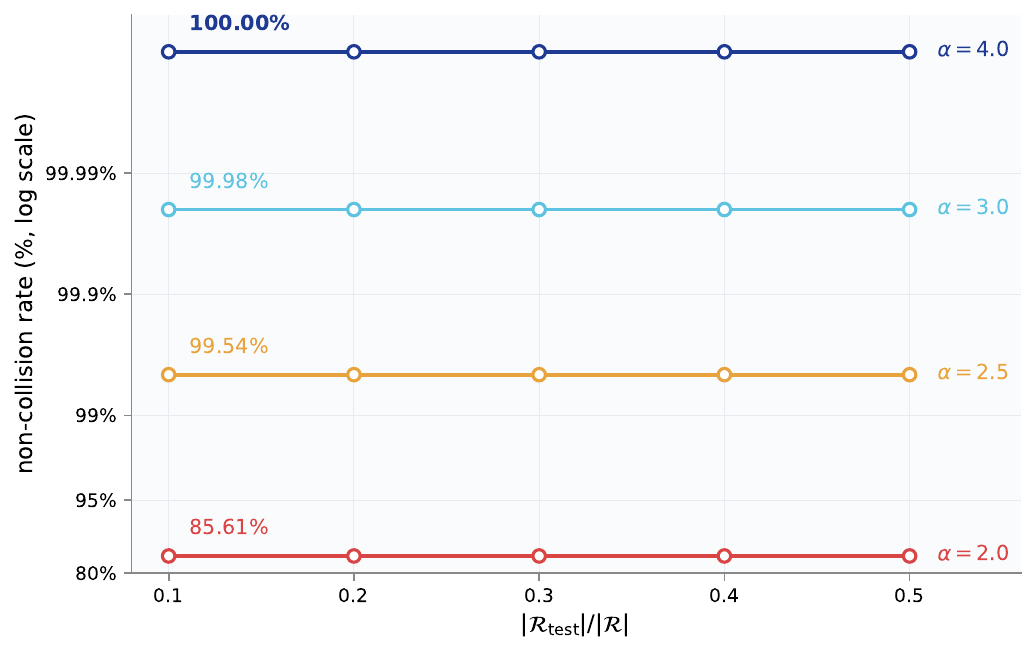}
\vspace{-6mm}
\caption{\small Open-world non-collision. At $\alpha{\geq}3.0$,
collision is near zero in all scales.}
\label{fig:openworld}
\vspace{-10pt}
\end{wrapfigure}
\Paragraph{Open-World Non-Collision as Real Galleries Grow.}
\label{sec:open-world}
The hard checks guarantee non-collision against $\mathcal{R}$
by construction, but provide no guarantee against real identities
enrolled after provisioning.
To stress-test this, we fix $|\mathcal{V}|$ at one million virtual
identities and progressively enroll $180$K held-out identities from
WebFace4M~\cite{zhu2021webface260m} (disjoint from $\mathcal{R}$),
measuring the non-collision rate as
$|\mathcal{R}_\text{test}|/|\mathcal{R}|$ grows from 0.1 to 0.5
($|\mathcal{R}|{=}360$K).
In embedding space, Fig.~\ref{fig:openworld} shows flat curves across
all $\alpha$ values at $\tau{=}0.391$, indicating that the collision
probability does not grow with test-gallery scale
(Appx.~\ref{app:empirical}); at our default $\alpha{=}4.0$,
collisions stay near zero while the gallery grows by half.
After realization the trend is the same at a lower level: for the
re-encoded $\phi(G(s))$, identity-level Non-Collision against the
full updated gallery moves from $98.61\%$ to $97.79\%$, so $0.83\%$ of deployed
identities require reassignment while $\mathcal{R}$ grows by half.
Provisioning at $\tau_\text{safe}{=}\tau{-}\Delta$ leaves every
identity a cosine margin of at least $\Delta$ to the enrolled gallery
(Proposition~\ref{prop:buffer}), which absorbs any realization or
capture shift that raises a gallery similarity by less than $\Delta$.
The margin costs capacity: at $\tau_\text{safe}{=}0.360$ the bound
falls from $7.41\times10^{10}$ to $1.81\times10^{9}$.

\Paragraph{Transfer to Other Encoders.}
The guarantee is encoder-specific (Remark~\ref{rem:galleryrelative}).
On the same 10K realized identities and the full gallery, two
AdaFace IR-101~\cite{kim2022adaface} models never used in
provisioning, each calibrated on its own IJB-B ROC at
FAR$\approx2\times10^{-5}$, give $99.59\%$ (WebFace4M) and $97.46\%$
(WebFace12M) Non-Collision, against $98.91\%$ for the allocation
encoder (Tab.~\ref{tab:encoder}).



\Paragraph{Cross-Reality Recognition on v-LFW Dataset.}
\label{sec:vlfw_eval}
A deployed system serving both real humans and digital entities
must simultaneously answer two questions for any face presented
to it: \emph{who is this?} and \emph{is this a real  or
 virtual identity?}
v-LFW is the first benchmark to answer both questions jointly,
mirroring LFW in identity count ($5{,}749$) and image count
($13{,}233$) and enabling four protocols, evaluated alongside the
real-image R-R baseline, on the combined LFW$+$v-LFW set.
%
%
\textbf{R-R} and \textbf{V-V} follow the standard LFW verification
protocol; \textbf{R-V} pairs are exclusively impostor pairs by BIP
construction; \textbf{Detection} distinguishes real from virtual
without identity labels; \textbf{Unified} simultaneously produces
a verified identity and a real-vs-virtual decision.
As shown in Tab.~\ref{tab:vlfw}, BIP + IAPCT is the only
method supporting all of them.
R-R Acc $99.80\%$ confirms that IAPCT leaves
recognition performance fully intact.
R-V FAR of $0.01\%$ shows that non-collision carries over to image
pairs on this benchmark, and Detection AUC $98.13\%$
confirms real and virtual populations are reliably separable
without modifying the frozen backbone.

\begin{table}[t]
\centering
\footnotesize
\caption{ 
\textbf{(a)} Five SoTA deepfake detectors~\cite{lin2025aiface} applied zero-shot to BIP-generated virtual faces.
\textbf{(b)} v-LFW evaluation on the combined LFW$+$v-LFW set. R-V pairs are exclusively impostor pairs by BIP construction. Unified denotes joint recognition and detection on the combined set. ``--'': task not supported by the method.}
\label{tab:safety_vlfw}
\vspace{1mm}

\begin{minipage}[t]{0.30\textwidth}
\centering
\phantomsubcaption\label{tab:safety}
\footnotesize
\setlength{\tabcolsep}{5pt}
\renewcommand{\arraystretch}{0.95}
\resizebox{0.75\linewidth}{!}{%
\begin{tabular}{lc}
\toprule
Detector & ACC~$\uparrow$ \\
\midrule
SPSL~\cite{liu2021spatial} & 87.76 \\
UCF~\cite{yan2023ucf} & 93.14 \\
Xception~\cite{rossler2019faceforensics++} & 95.13 \\
EfficientB4~\cite{tan2019efficientnet} & 97.13 \\
PG-FDD~\cite{lin2024preserving} & \textbf{99.57} \\
\bottomrule
\end{tabular}}
\vspace{0.5mm}

\textbf{(\subref{tab:safety})}
\end{minipage}
\hfill
\begin{minipage}[t]{0.68\textwidth}
\centering
\phantomsubcaption\label{tab:vlfw}
\footnotesize
\setlength{\tabcolsep}{5pt}
\renewcommand{\arraystretch}{1.00}
\resizebox{0.99\linewidth}{!}{%
\begin{tabular}{lccc}
\toprule
Metric & DCFace~\cite{kim2023dcface}+AdaFace & Vec2Face~\cite{wu2024vec2face}+ArcFace & BIP+IAPCT (Ours) \\
\midrule
R-R Acc $\uparrow$   & 98.58 & 98.87 & \textbf{99.80} \\
V-V Acc $\uparrow$   & 98.85 & 99.04 & \textbf{99.18} \\
R-V FAR $\downarrow$ & 0.05 & 0.02 & \textbf{0.01} \\
Det. AUC $\uparrow$  & --    & --    & \textbf{98.13} \\
Unified $\uparrow$   & --    & --    & \textbf{99.20} \\
\bottomrule
\end{tabular}}
\vspace{0.5mm}

\textbf{(\subref{tab:vlfw})}
\end{minipage}
\end{table}


\section{Conclusion}
\label{sec:conclu}

We introduced Biometric Identity Provisioning (BIP), the first
formal framework for allocating non-colliding biometric identities
to digital entities at scale.
The key geometric insight is that valid virtual identities must
occupy unclaimed gaps within the real face manifold rather than
a separate free subspace. This finding shapes repulsion-based
allocation, gap-aware realization via GapGen, and the capacity
analysis, which leaves about $7{,}400\times$ headroom at the 10M
scale under the submanifold model.
Against a gallery of 360K real identities, 10M
non-colliding virtual identity embeddings are provisioned with
no observed collision, and one million are realized as
high-fidelity face images, of which $98.61\%$ remain non-colliding
after re-encoding.
Our v-LFW dataset benchmarks real-virtual coexistence: R-V FAR of
$0.01\%$ shows that non-collision carries over to image pairs, and
existing deepfake detectors flag the generated faces zero-shot with
$87.76$--$99.57\%$ accuracy.

\Paragraph{Limitations.}
BIP guarantees non-collision against $\mathcal{R}$ at provisioning
time; under continued enrollment the identity is long-lived and
lifecycle-managed rather than permanent, and the safety buffer
(Appx.~\ref{app:safety_buffer}) trades capacity for margin.
The guarantee is exact in embedding space; after realization,
separability at scale is bounded by the round-trip fidelity of the
generator, which degrades at large $\alpha$, and applying the hard
checks to $\phi(G(s))$ is a direct next step.
Capacity bounds rest on a spherical submanifold approximation
rather than a precise characterisation of the face manifold
(Appx.~\ref{app:capacity}).
Allocation is specific to the ArcFace ResNet-100 space: two AdaFace
encoders retain $97.46$--$99.59\%$ Non-Collision, and joint
multi-encoder provisioning is a natural extension.
Robustness to physical display and re-capture remains to be studied,
and the demographic composition of provisioned identities is not
controlled and may differ from that of the gallery.

\begin{ack}
This work was supported in part by the National Science Foundation under Grant FMitF-2319242 and by the NVIDIA Academic Grant Program through an award to Feng Liu.
The authors declare no competing interests.
\end{ack}

\newpage
\bibliographystyle{plain}
\bibliography{refs_CR}


\clearpage
\appendix

\section*{Appendix}

This appendix provides full details supporting the main paper.

\begin{itemize}

  \item Sec.~\ref{app:broader}: Broader Impact and Ethics.

  \item Sec.~\ref{app:threshold}: Verification threshold selection
        and IJB-B operating points.

  \item Sec.~\ref{app:capacity}: Capacity bound derivations.
  \begin{itemize}
    \item Sec.~\ref{app:pca}: PCA analysis and submanifold model.
    \item Sec.~\ref{app:spherical}: Spherical geometry preliminaries
          and GV lower bound.
    \item Sec.~\ref{app:gv}: Gilbert--Varshamov lower bound.
    \item Sec.~\ref{app:effective_capacity}: Effective capacity,
          minimum separation threshold (Corollary~\ref{cor:alpha}),
          and sensitivity to $\tau$.
    \item Sec.~\ref{app:repulsion}: Repulsion direction analysis.
    \item Sec.~\ref{app:safety_buffer}: Safety buffer
          (Proposition~\ref{prop:buffer}) and lifecycle policy.
    \item Sec.~\ref{app:empirical}: Empirical capacity estimator.
    \item Sec.~\ref{app:alpha_kappa}: Capacity, $\alpha$, and
          PCA noise.
  \end{itemize}

  \item Sec.~\ref{app:gapgen}: Additional implementation Details of $G$ (GapGen).
  \begin{itemize}
    \item Base pipeline, conditioning, and sampling.
    \item Round-trip re-encoding.
    \item Gap-aware fine-tuning.
  \end{itemize}

  \item Sec.~\ref{app:iapct}: Verification, Visual Examples, and IAPCT.
  \begin{itemize}
    \item Sec.~\ref{app:vlfw}: v-LFW protocol and visual examples.
    \item Sec.~\ref{app:iapct-arch}: IAPCT architecture,
          tokenisation, and training details.
  \end{itemize}
  \item Sec.~\ref{app:add}: Additional results.
  \begin{itemize}
    \item Sec.~\ref{app:add-tsne}: t-SNE of $\mathcal{R}$ vs.\ $\mathcal{V}$.
    \item Sec.~\ref{app:add-grid}: Additional image grids.
    \item Sec.~\ref{app:realized_capacity}: Realized capacity.
    \item Sec.~\ref{app:encoder}: Encoder transfer.
  \end{itemize}
\end{itemize}


\section{Broader Impact and Ethics}\label{app:broader}
BIP provisions biometric identities for digital entities, not
synthetic faces of real people.
The non-collision constraint ensures provisioned identities lie
outside all enrolled real identity territories, and
Sec.~\ref{sec:safety} confirms they are detectable by existing
deepfake detectors without modification.
Potential misuse to evade biometric access control is mitigated
by the gallery-relative design: an adversary requires access to
the full enrollment gallery.
All training data are used under their respective licenses, and
v-LFW contains no real person's likeness.

\section{Verification Threshold Selection}
\label{app:threshold}

The BIP non-collision constraint $\cos(v_j, c_i) < \tau$ and all
capacity bounds in Appendix~\ref{app:capacity} are stated in terms
of a verification threshold $\tau$ that must be grounded in the
operating point of the deployed recognition system rather than set
arbitrarily.

\Paragraph{Protocol.}
We evaluate ArcFace ResNet-100~\cite{deng2019arcface} (Glint360K, \texttt{antelopev2/glintr100},
ONNX)
on IJB-B~\cite{whitelam2017iarpa} under the 1:1 verification protocol
across a sweep of FAR operating points from $10^{-5}$ to $10^{-4}$.
The ROC curve and operating points are shown in Fig.~\ref{fig:roc}.

\begin{figure}[t]
  \centering
  \includegraphics[width=0.85\linewidth]{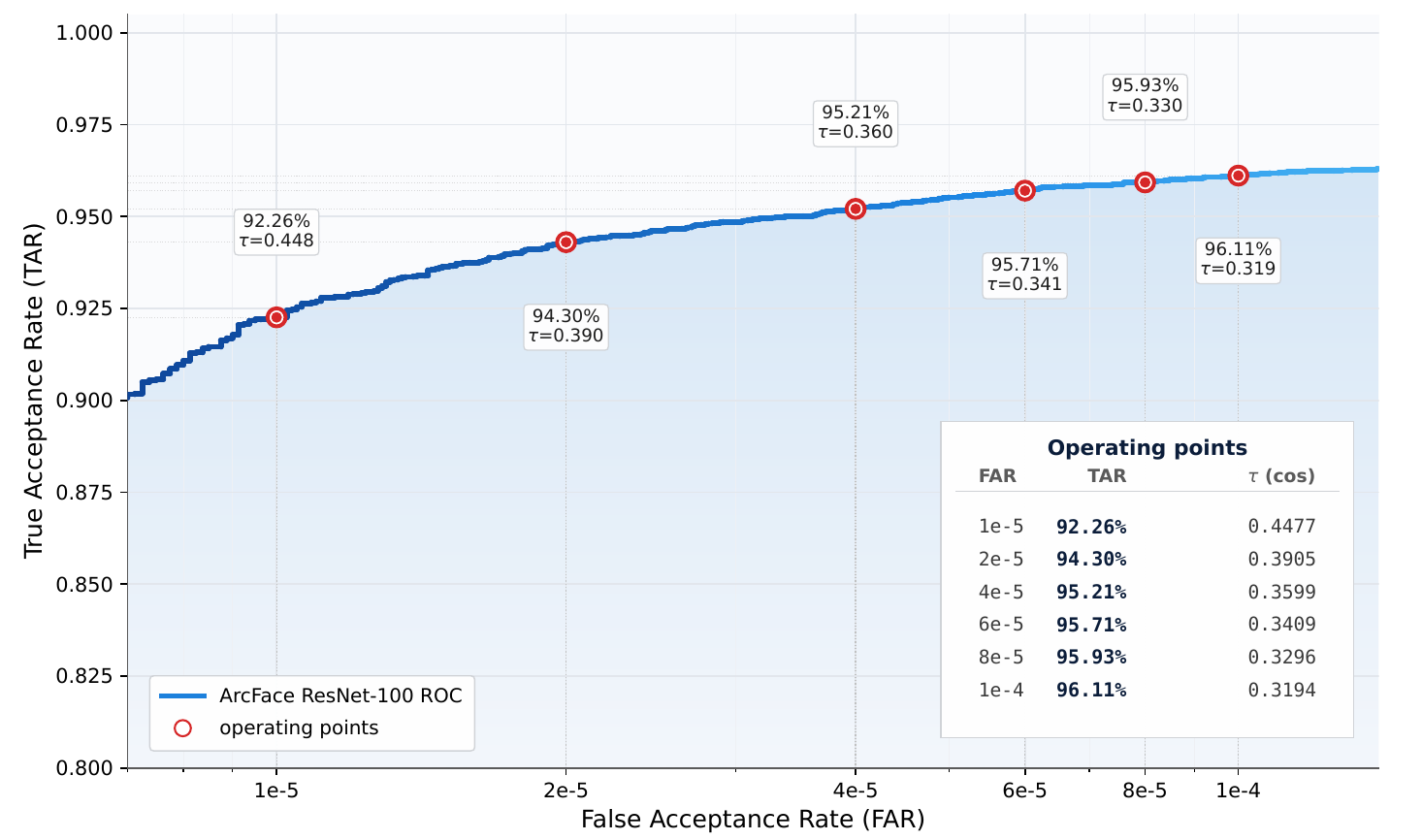}
  \caption{ROC curve for ArcFace ResNet-100 on IJB-B 1:1
  verification (8.01M pairs; 10,270 genuine / 8.00M impostor;
  AUC $= 99.49\%$).
  Red circles mark the six operating points used as BIP threshold
  candidates; $\tau$ denotes the cosine similarity threshold
  achieving the stated FAR.}
  \label{fig:roc}
\end{figure}

\Paragraph{Threshold Choice and BIP Implications.}
The choice of $\tau$ directly governs the stringency of the BIP
non-collision constraint and the available face manifold capacity:

\begin{itemize}[leftmargin=*]
  \item \textbf{Larger $\tau$ (\emph{e.g.},
        FAR $= 10^{-5}$, $\tau = 0.448$):}
        fewer impostor pairs are falsely accepted.
      The acceptance cap $\{x:\cos(x,c_i)\geq\tau\}$ around
      each real identity is \emph{smaller}, and the BIP
      non-collision constraint $\cos(v_j,c_i)<\tau$ is looser.
  \item \textbf{Smaller $\tau$ (\emph{e.g.},
        FAR $= 10^{-4}$, $\tau = 0.319$):}
      more impostor pairs are accepted.
      The acceptance cap around each real identity is
      \emph{larger}, and the BIP non-collision constraint
      is stricter, requiring more perturbation.
\end{itemize}

We use $\tau = 0.391$ as the primary operating point in all main
experiments, corresponding to FAR $\approx 2\times10^{-5}$
and TAR $= 94.30\%$ on IJB-B.
At this operating point, one in $5\times10^4$ impostor pairs is
falsely accepted, a conservative threshold appropriate for
high-security deployments.
We ablate BIP metrics across all six operating points
in Tab.~\ref{tab:alpha_threshold_main} of the main paper.

\section{Capacity Bound Derivation}
\label{app:capacity}

This section provides full derivations supporting the BIP
capacity analysis in Sec.~\ref{sec:geometry}.
Throughout, $\mathbb{S}^{d-1}$ denotes the unit hypersphere in
$\mathbb{R}^d$ with $d=512$, $\tau\in(0,1)$ is the verification
threshold of Assumption~\ref{ass:encoder}, and $\tilde{d}=269$ is the effective manifold dimensionality
established in Sec.~\ref{app:pca} below.
All reported numerical bounds use the \emph{exact} regularized
incomplete beta function; the Gaussian $Q$-function is used
for intuition only.

\Paragraph{Notation.}
We distinguish three related quantities used throughout:
\begin{itemize}[leftmargin=*]
  \item $A_{\mathrm{pack}}(\tilde{d},\tau)$: the (unknown)
        maximum size of a $\tau$-code on the model sphere
        $\mathbb{S}^{\tilde{d}-1}$ that stands in for
        $\mathcal{M}_\text{face}$ (Sec.~\ref{app:pca}).
  \item $A_{\mathrm{GV}}(\tilde{d},\tau) := 1/\mu(\tau,\tilde{d})$:
        the computable Gilbert--Varshamov lower bound, satisfying
        $A_{\mathrm{pack}} \geq A_{\mathrm{GV}}$.
  \item $A_{\mathrm{eff}}(\tau) := 1/p_{\mathrm{coll}}(\tau)$:
        the empirical effective capacity, estimated from observed
        collision counts between $\mathcal{V}$ and unseen real
        identities.
\end{itemize}
Since $\mathcal{M}_\text{face}$ is approximated by
$\mathbb{S}^{\tilde{d}-1}$ rather than equal to it, $A_{\mathrm{GV}}$
is a model-based capacity scale for the face manifold and not a
rigorous bound on its packing number.
The operational capacity statements of the paper are empirical:
allocation at 10M identities (Tab.~\ref{tab:alpha_threshold_main})
and realized capacity (Appx.~\ref{app:realized_capacity}).

\subsection{Real Face Manifold: PCA and Submanifold Model}\label{app:pca}

Fig.~\ref{fig:pca_appendix} visualizes the PCA structure of
Glint360K identity centroids, confirming Observation~\ref{prop:manifold}:
95\% of variance concentrates within $\tilde{d}{=}269$ components
out of $d{=}512$.
We therefore model $\mathcal{M}_\text{face}$ as locally approximated
by $\mathbb{S}^{\tilde{d}-1}$ and apply all packing bounds within
this reduced sphere.
This is a \emph{capacity-scale approximation}: the cap angle
$\arccos(0.391){\approx}67^\circ$ is not a small local neighborhood,
so results should be interpreted as headroom estimates rather
than precise geometric statements.
%

\begin{figure}[t]
  \centering
  \includegraphics[width=\linewidth]{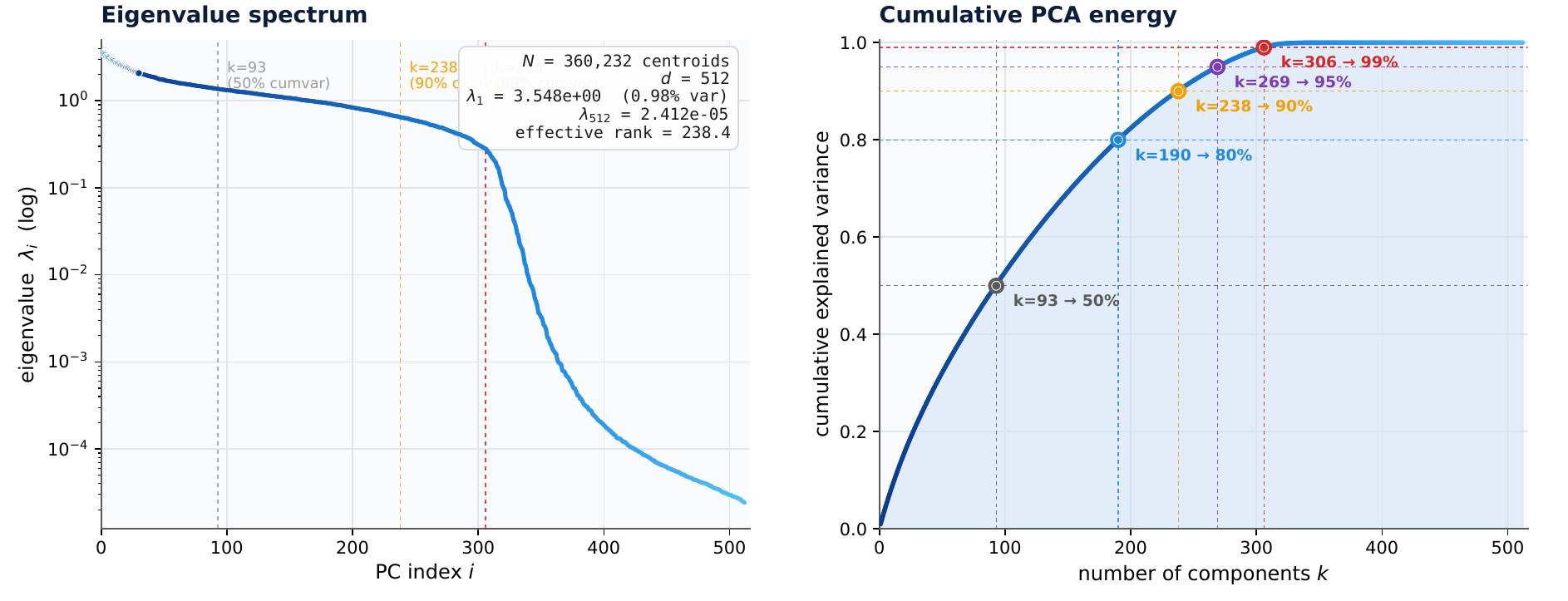}
  \caption{\textbf{PCA of Glint360K identity centroids}
  ($N{=}360{,}232$, ArcFace ResNet-100, $d{=}512$).
  \textbf{Left (eigenvalue spectrum):}
  $\lambda_k$ decays sharply after $k{\approx}300$, with
  $\lambda_{512} = 2.4{\times}10^{-5}$ five orders of magnitude
  below $\lambda_1 = 3.55$; effective rank $= 238.4$.
  \textbf{Right (cumulative explained variance):}
  50\% of variance is captured within $k{=}93$ components,
  90\% within $k{=}238$, 95\% within $k{=}269$, and 99\%
  within $k{=}306$.}
  \label{fig:pca_appendix}
\end{figure}

\subsection{Spherical Geometry Preliminaries}
\label{app:spherical}

\Paragraph{Spherical Cap Volume.}
The \emph{spherical cap} of cosine radius $\tau$ centered at
$p \in \mathbb{S}^{d-1}$ is:
\begin{equation}
  \mathrm{Cap}(p,\tau)
  = \bigl\{x \in \mathbb{S}^{d-1} : \cos(x, p) \geq \tau\bigr\}.
\end{equation}
Its normalized volume is independent of $p$ by rotational symmetry:
\begin{equation}
  \mu(\tau, d)
  = \frac{\mathrm{Vol}(\mathrm{Cap}(p,\tau))}
         {\mathrm{Vol}(\mathbb{S}^{d-1})}
  = \frac{1}{2}\,
    I_{1-\tau^2}\!\!\left(\frac{d-1}{2},\;\frac{1}{2}\right),
  \label{eq:cap_vol}
\end{equation}
where $I_x(a,b)$ is the regularized incomplete beta function.

\Paragraph{Verification of Eq.~\eqref{eq:cap_vol}.}
For $\mathbb{S}^1$ (circle in $\mathbb{R}^2$) with $\tau = 0.5$:
$\mu = \frac{1}{2}I_{0.75}(\frac{1}{2}, \frac{1}{2})
= \frac{1}{\pi}\arcsin(\sqrt{0.75}) = \frac{1}{3}$,
matching the $120^\circ/360^\circ$ arc fraction for a $60^\circ$ half-angle
cap. $\checkmark$

\Paragraph{Monotonicity.}
$\mu(\tau, d)$ is \emph{strictly decreasing} in $\tau$: larger
$\tau$ defines a smaller cap.
$\mu(\tau, d)$ is also strictly decreasing in $d$ by
concentration of measure: for $\tilde{d} < d$,
$\mu(\tau, \tilde{d}) \geq \mu(\tau, d)$.

\Paragraph{Gaussian Approximation (Intuition Only).}
$\mu(\tau, d) \approx Q(\tau\sqrt{d})$ by the CLT, but in the
far-tail regime this substantially overestimates the exact value.
At $\tau = 0.391$, $d = 269$: the Gaussian gives
$Q(0.391\sqrt{269}) \approx Q(6.41) \approx 7.2\times10^{-11}$
versus the exact $\mu = 1.35\times10^{-11}$ (a $5.3\times$
overestimate).
All reported bounds use Eq.~\eqref{eq:cap_vol} directly.

\Paragraph{Packing Number.}
A \emph{spherical $\tau$-code} is a set
$\mathcal{S} \subset \mathbb{S}^{d-1}$ with
$\cos(s_i, s_j) < \tau$ for all $i \neq j$.
Its maximum cardinality is the packing number:
\begin{equation}
  A(d, \tau)
  = \max\bigl\{|\mathcal{S}| :
    \mathcal{S} \subset \mathbb{S}^{d-1},\;
    \cos(s_i,s_j) < \tau\;\forall i\neq j\bigr\},
  \label{eq:packing_number}
\end{equation}
which equals the BIP capacity of Definition~\ref{def:bip}
when $\mathcal{R} = \emptyset$.

\subsection{Gilbert--Varshamov Lower Bound}
\label{app:gv}

\begin{theorem}[Gilbert--Varshamov Bound]
\label{thm:gv}
\begin{equation}
  A(d,\tau) \;\geq\; \frac{1}{\mu(\tau,d)} = A_{\mathrm{GV}}(d,\tau).
  \label{eq:gv}
\end{equation}
\end{theorem}

\begin{proof}
Let $\mathcal{S}^* \subset \mathbb{S}^{d-1}$ be a maximal
$\tau$-code, \emph{i.e.}, no additional point can be appended without
violating the separation constraint.
Maximality implies the caps
$\{\mathrm{Cap}(s,\tau)\}_{s \in \mathcal{S}^*}$
cover all of $\mathbb{S}^{d-1}$; otherwise an uncovered point
could be appended, contradicting maximality.
Comparing total cap volume to sphere volume:
\begin{equation}
  |\mathcal{S}^*| \cdot \mu(\tau, d) \geq 1
  \;\implies\;
  |\mathcal{S}^*| \geq \frac{1}{\mu(\tau,d)}.
\end{equation}
Since $\mathcal{S}^*$ is itself a valid $\tau$-code and
$A(d,\tau)$ is the maximum cardinality over all such codes:
  $A(d,\tau) \geq |\mathcal{S}^*| \geq \frac{1}{\mu(\tau,d)}$.
\end{proof}

\Paragraph{Numerical Evaluation at $\tau=0.391$, $d=512$.}
\begin{equation}
  \mu(0.391,512) \approx 1.74\times10^{-20},
  \quad
  A_{\mathrm{GV}}(512,0.391) \approx 5.75\times10^{19}
  \approx 2^{65.6}.
\end{equation}
This full-hypersphere bound is not directly useful for BIP since
real face identities occupy only a $\tilde{d}$-dimensional
submanifold of $\mathbb{S}^{511}$.

\subsection{Effective Capacity and GV Lower Bound}
\label{app:effective_capacity}
Applying Theorem~\ref{thm:gv} to the $\mathbb{S}^{\tilde{d}-1}$
submanifold model of Sec.~\ref{app:pca} yields the following
capacity lower bound.

\begin{proposition}[BIP Capacity --- GV Lower Bound]
\label{prop:capacity}
Under the $\mathbb{S}^{\tilde{d}-1}$ submanifold model:
\begin{equation}
  A_{\mathrm{pack}}(\tilde{d},\tau)
  \;\geq\;
  A_{\mathrm{GV}}(\tilde{d},\tau)
  = \frac{1}{\mu(\tau,\tilde{d})}
  = \frac{2}{I_{1-\tau^2}\!\left(\dfrac{\tilde{d}-1}{2},
    \dfrac{1}{2}\right)}.
  \label{eq:capacity}
\end{equation}
For $\tilde{d}=269$ and $\tau=0.391$,
$\mu(0.391,269)=1.35\times10^{-11}$, giving:
\begin{equation}
  A_{\mathrm{GV}}(269,0.391)
  = 7.41\times10^{10}
  \approx 2^{36.1}.
  \label{eq:capacity_numeric}
\end{equation}
\end{proposition}

\begin{proof}
Apply Theorem~\ref{thm:gv} to dimension $\tilde{d}$ in place
of $d$.
\end{proof}


\begin{corollary}[Minimum Separation Threshold]
\label{cor:alpha}
For $r, z \in \mathbb{S}^{d-1}$ with $p = r \cdot z$ and
$|p| < \tau$, let
$s = \mathrm{normalize}(r + \alpha z)$.
The unique $\alpha^* > 0$ at which $\cos(s, r) = \tau$
exactly is:
\begin{equation}
  \alpha^*(p, \tau)
  = \frac{\sqrt{1-\tau^2}\!\left[
      p\sqrt{1-\tau^2} + \tau\sqrt{1-p^2}
    \right]}{\tau^2 - p^2}.
  \label{eq:alphastar_general}
\end{equation}
For $\alpha > 0$ and $|p| < 1$, $\cos(s, r)$ is strictly
decreasing in $\alpha$:
\begin{equation}
  \frac{d}{d\alpha}
  \frac{1 + \alpha p}{\sqrt{1 + 2\alpha p + \alpha^2}}
  = \frac{\alpha(p^2 - 1)}{(1 + 2\alpha p + \alpha^2)^{3/2}}
  < 0,
  \label{eq:cosdecrease}
\end{equation}
so $\cos(s, r) < \tau$ requires $\alpha > \alpha^*(p, \tau)$
strictly.

In the orthogonal case $p = 0$, Eq.~\eqref{eq:alphastar_general}
simplifies to:
\begin{equation}
  \alpha^*(0, \tau)
  = \frac{\sqrt{1-\tau^2}}{\tau}.
  \label{eq:alphastar_orth}
\end{equation}
For the operating thresholds used in this work
($\tau \leq 0.5 < 1/\sqrt{2}$), $\alpha^*(0,\tau)$ is the
maximum of $\alpha^*(p,\tau)$ over the repulsive range
$p \in [-\tau, 0]$.
Therefore, $\alpha > \sqrt{1-\tau^2}/\tau$ is sufficient for
$\cos(s,r) < \tau$ over the repulsive range $p\in[-\tau,0]$.

\textbf{Note on positive $p$:}
With PCA-aware noise (Eq.~\eqref{eq:perturbation}), $p=r\cdot z$
can be positive. For $p>0$, $\alpha^*$ grows rapidly:
$\alpha^*(0.3,0.391)\approx9.5$, far exceeding the orthogonal
bound $2.35$.
The hard checks of Eqs.~\eqref{eq:hardcheck1}--\eqref{eq:hardcheck2}
enforce BIP constraints irrespective of the sign of $r\cdot z$.
\end{corollary}

\begin{proof}
\textbf{Deriving $\alpha^*(p,\tau)$.}
From Lemma~\ref{lem:displacement}:
\begin{equation}
  \cos(s, r)
  = \frac{1 + \alpha p}{\sqrt{1 + 2\alpha p + \alpha^2}}.
  \label{eq:cos_sr_proof}
\end{equation}
Setting $\cos(s,r) = \tau > 0$: the right-hand side equals
$\tau > 0$, so the numerator $1 + \alpha p$ must be positive
at any solution (since the denominator is always positive).
Squaring the equality at the solution is therefore valid:
\begin{equation}
  \tau^2(1 + 2\alpha p + \alpha^2) = (1 + \alpha p)^2,
\end{equation}
which rearranges to the quadratic:
\begin{equation}
  \alpha^2(\tau^2 - p^2)
  + 2\alpha p(\tau^2 - 1)
  + (\tau^2 - 1) = 0.
  \label{eq:quadratic_proof}
\end{equation}
Since $|p| < \tau$, we have $\tau^2 - p^2 > 0$, so
Eq.~\eqref{eq:quadratic_proof} is a well-posed quadratic.
Its discriminant is:
\begin{align}
  \Delta
  &= 4p^2(\tau^2-1)^2 - 4(\tau^2-p^2)(\tau^2-1)
   = 4(\tau^2-1)\bigl[p^2(\tau^2-1) - (\tau^2-p^2)\bigr]
   \notag\\
  &= 4(\tau^2-1)\tau^2(p^2-1)
   = 4\tau^2(1-\tau^2)(1-p^2) > 0,
  \label{eq:discriminant_proof}
\end{align}
giving $\sqrt{\Delta} = 2\tau\sqrt{(1-\tau^2)(1-p^2)}$.
The positive root of~\eqref{eq:quadratic_proof} is:
\begin{equation}
  \alpha^*(p,\tau)
  = \frac{p(1-\tau^2) + \tau\sqrt{(1-\tau^2)(1-p^2)}}
         {\tau^2 - p^2}
  = \frac{\sqrt{1-\tau^2}\bigl[p\sqrt{1-\tau^2}
    + \tau\sqrt{1-p^2}\bigr]}{\tau^2 - p^2},
\end{equation}
establishing Eq.~\eqref{eq:alphastar_general}.
At $p = 0$: $\alpha^*(0,\tau) = \tau\sqrt{1-\tau^2}/\tau^2
= \sqrt{1-\tau^2}/\tau$, confirming Eq.~\eqref{eq:alphastar_orth}.

\textbf{Strict inequality Eq.~\eqref{eq:cosdecrease}.}
Direct differentiation of
$f(\alpha) = (1+\alpha p)/\sqrt{1 + 2\alpha p + \alpha^2}$:
\begin{equation}
  f'(\alpha)
  = \frac{p\sqrt{1+2\alpha p+\alpha^2}
         - (1+\alpha p)\cdot\dfrac{p+\alpha}{\sqrt{1+2\alpha p+\alpha^2}}}
         {1+2\alpha p+\alpha^2}
  = \frac{\alpha(p^2-1)}{(1+2\alpha p+\alpha^2)^{3/2}}.
\end{equation}
Since $\alpha > 0$ and $p^2 < 1$ (as $|p| < \tau < 1$),
$f'(\alpha) < 0$, establishing Eq.~\eqref{eq:cosdecrease}.
Since $f(0) = 1 > \tau$ and
$\lim_{\alpha\to\infty} f(\alpha) = p < \tau$
(as $|p| < \tau$), strict monotonicity of $f$ implies
the solution $\alpha^*(p,\tau)$ is unique.

\textbf{Upper bound at $\tau \leq 1/\sqrt{2}$.}
By the implicit function theorem applied to
$f(\alpha^*(p), p) = \tau$, with $f_\alpha < 0$
(from Eq.~\eqref{eq:cosdecrease}) and
$f_p = \alpha^2(\alpha+p)/(1+2\alpha p+\alpha^2)^{3/2}$:
\begin{equation}
  \frac{d\alpha^*}{dp}
  = -\frac{f_p}{f_\alpha}
  = \frac{\alpha^*(\alpha^* + p)}{1 - p^2}.
  \label{eq:dalpha_dp}
\end{equation}
The sign of~\eqref{eq:dalpha_dp} depends on $\alpha^*(p,\tau)+p$.
Taking the limit $p \to -\tau^+$ using
Eq.~\eqref{eq:alphastar_general}:
\begin{equation}
  \lim_{p \to -\tau^+}(\alpha^*(p,\tau) + p)
  = \frac{1}{2\tau} - \tau
  = \frac{1 - 2\tau^2}{2\tau}
  \;\geq\; 0
  \quad\iff\quad \tau \leq \frac{1}{\sqrt{2}}.
  \label{eq:boundary_limit}
\end{equation}
At $p = 0$: $\alpha^*(0,\tau) + 0 = \sqrt{1-\tau^2}/\tau > 0$.
Since $\alpha^*>0$ and $1-p^2>0$, the sign of
$d\alpha^*/dp$ is determined by $\alpha^*+p$.
Using Eq.~\eqref{eq:alphastar_general}:
\begin{equation}
  \alpha^*(p,\tau)+p
  = \frac{\sqrt{1-p^2}\bigl[p\sqrt{1-p^2}+\tau\sqrt{1-\tau^2}\bigr]}
         {\tau^2-p^2}.
\end{equation}
The denominator $\tau^2-p^2>0$ for $|p|<\tau$, and
$\sqrt{1-p^2}>0$, so the sign reduces to
$p\sqrt{1-p^2}+\tau\sqrt{1-\tau^2}$.
Define $h(p)=p\sqrt{1-p^2}$; then
$h'(p)=(1-2p^2)/\sqrt{1-p^2}\geq0$
for $p^2\leq\tau^2\leq1/2$ (since $\tau\leq1/\sqrt{2}$),
so $h$ is non-decreasing on $[-\tau,0]$ and
$h(p)\geq h(-\tau)=-\tau\sqrt{1-\tau^2}$.
Therefore $p\sqrt{1-p^2}+\tau\sqrt{1-\tau^2}\geq0$,
giving $\alpha^*(p,\tau)+p\geq0$ and $d\alpha^*/dp\geq0$
throughout $p\in[-\tau,0]$.
Hence $\alpha^*(0,\tau)=\sqrt{1-\tau^2}/\tau$ is the
maximum over the repulsive range.
\end{proof}

\Paragraph{Sensitivity to $\tau$.}
Table~\ref{tab:capacity} reports exact GV bounds and the
sufficient perturbation threshold.
Since $\mu(\tau,\tilde{d})$ is strictly decreasing in $\tau$,
larger $\tau$ corresponds to a less restrictive non-collision
constraint (smaller excluded cap per identity), yielding a
larger $A_\text{GV}$.
The safety buffer requires provisioning at
$\tau_\text{safe} < \tau$ (stricter constraint), which
\emph{reduces} $A_\text{GV}$ and \emph{increases} the required
perturbation; see Proposition~\ref{prop:buffer}.

\begin{table}[h]
\centering
\small
\caption{GV lower bound $A_\text{GV}=1/\mu(\tau,269)$ at the
six IJB-B operating points (Fig.~\ref{fig:roc}),
computed via exact regularized incomplete beta.
$\alpha^*(0,\tau)=\sqrt{1-\tau^2}/\tau$ is the orthogonal
worst-case sufficient perturbation threshold over repulsive
directions $p\in[-\tau,0]$, valid for $\tau\leq1/\sqrt{2}$.
\textbf{Bold}: primary operating point ($\tau=0.391$,
FAR$\approx2\times10^{-5}$).}
\label{tab:capacity}
\begin{tabular}{ccccc}
\toprule
$\tau$ & $\mu(\tau,269)$ &
$\log_2 A_\text{GV}$ &
$A_\text{GV}=1/\mu$ &
$\alpha^*(0,\tau)$ \\
\midrule
0.319 & $4.21\times10^{-8\phantom{0}}$ & 24.50 & $2.38\times10^{7\phantom{00}}$ & 2.97 \\
0.330 & $1.40\times10^{-8\phantom{0}}$ & 26.09 & $7.15\times10^{7\phantom{00}}$ & 2.86 \\
0.341 & $4.45\times10^{-9\phantom{0}}$ & 27.74 & $2.25\times10^{8\phantom{00}}$ & 2.76 \\
0.360 & $5.52\times10^{-10}$           & 30.75 & $1.81\times10^{9\phantom{00}}$ & 2.59 \\
\textbf{0.391} & $\mathbf{1.35\times10^{-11}}$ & \textbf{36.11} & $\mathbf{7.41\times10^{10}}$ & \textbf{2.35} \\
0.448 & $4.92\times10^{-15}$           & 47.53 & $2.03\times10^{14}$           & 2.00 \\
\bottomrule
\end{tabular}
\end{table}

\Paragraph{Model-Based Headroom.}
At $\tau=0.391$ and $|\mathcal{V}|=1$M:
\begin{equation}
  \frac{A_\text{GV}(269,\,0.391)}{|\mathcal{V}|}
  = \frac{7.41\times10^{10}}{10^6}
  \approx 7.41\times10^4
  \approx 2^{16.2}.
\end{equation}
The GV lower bound exceeds one million provisioned identities
by a factor of $\approx74{,}000$.
For the 10M embedding-allocation experiment:
\begin{equation}
  \frac{A_\text{GV}(269,\,0.391)}{10^7}
  = \frac{7.41\times10^{10}}{10^7}
  \approx 7.41\times10^3.
\end{equation}
Even the GV lower bound leaves $\approx7{,}400\times$
headroom at the largest tested allocation scale, indicating
that BIP provisioning remains well within the model-based
face manifold capacity throughout all reported experiments.

\subsection{Repulsion Direction: Heuristic Motivation and
Formal Guarantee}
\label{app:repulsion}

Remark~\ref{rem:repulsion} in the main text introduces
$z^*$ as a repulsion heuristic.
We provide the precise mathematical analysis here.

\Paragraph{What is proven.}
With $m = \sum_{k=1}^K w_{n_k} c_{n_k}$ and
$z^* = -m/\|m\|_2$:
\begin{equation}
  z^* \cdot c_{n_k} < 0
  \quad \text{for any } c_{n_k} \text{ satisfying }
  c_{n_k}^\top m > 0.
\end{equation}
Neighbors positively aligned with the weighted centroid $m$
satisfy this condition; whether \emph{all} neighbors do
depends on the local geometry and is not guaranteed solely
by the nearest-neighbour construction.

\Paragraph{What is not proven after normalization.}
The relevant quantity for BIP is the rate of change of cosine
similarity to a neighbor under the spherically-normalized
perturbation $s(\alpha) = \mathrm{normalize}(r + \alpha z^*)$.
By the chain rule:
\begin{equation}
  \frac{d}{d\alpha}\cos(s(\alpha), c_{n_k})\bigg|_{\alpha=0}
  = z^* \cdot c_{n_k} - (r \cdot z^*)(r \cdot c_{n_k}).
  \label{eq:deriv_neighbor}
\end{equation}
Even when $z^* \cdot c_{n_k} < 0$, the second term
$(r \cdot z^*)(r \cdot c_{n_k})$ can dominate:
if $r \cdot z^* < 0$ (typical for a repulsion direction) and
$r \cdot c_{n_k} > 0$ (neighbor close to $r$), their product
is negative and contributes positively to the derivative.
Whether Eq.~\eqref{eq:deriv_neighbor} is negative, \emph{i.e.},
whether moving along $z^*$ genuinely increases cosine
distance to $c_{n_k}$ after renormalization --- depends on
the local geometry and is not guaranteed by the construction
of $z^*$ alone.

\Paragraph{Correct status.}
$z^*$ is a \emph{repulsive heuristic}: it improves the
probability that a candidate passes the BIP hard checks,
but does not formally guarantee it.
The formal non-collision guarantee comes exclusively from
Eqs.~\eqref{eq:hardcheck1}--\eqref{eq:hardcheck2}, which
accept a candidate $s$ only after exact cosine verification
against all of $\mathcal{R}$ and $\mathcal{V}_t$.

    \subsection{Safety Buffer for Open-World Robustness}
    \label{app:safety_buffer}

\begin{proposition}[Safety Buffer]
\label{prop:buffer}
Let $\tau$ be the operating verification threshold
(collision when $\cos \geq \tau$) and let
$\tau_{\mathrm{safe}} = \tau - \Delta$ with $\Delta > 0$.
If $\mathcal{V}$ is provisioned such that
$\cos(v_j, c_i) < \tau_{\mathrm{safe}}$
for all $v_j \in \mathcal{V}$, $c_i \in \mathcal{R}$, then:
\begin{enumerate}[label=(\roman*),leftmargin=*]
  \item Every provisioned identity has a cosine margin of
        at least $\Delta$ below the operating collision
        boundary $\tau$ with respect to all enrolled gallery
        identities at provisioning time.
  \item For a subsequently enrolled real identity $r'$,
        any $v_j$ satisfying
        $\cos(v_j, r') \in [\tau_{\mathrm{safe}}, \tau)$
        does \emph{not} cause an operational collision
        (since $\cos(v_j, r') < \tau$), but lies within the
        monitoring zone and may trigger reassessment under
        the provisioning policy.
        Only $\cos(v_j, r') \geq \tau$ constitutes a
        collision requiring revocation.
  \item The safety buffer requires \emph{more} perturbation
        than provisioning at $\tau$: since
        $\tau_{\mathrm{safe}} < \tau$ and
        $d\alpha^*(0,\tau)/d\tau = -1/(\tau^2\sqrt{1-\tau^2}) < 0$,
        it follows that
        $\alpha^*(0, \tau_{\mathrm{safe}}) > \alpha^*(0, \tau)$.
        Concretely, $\tau_{\mathrm{safe}} = 0.360$,
        $\tau = 0.391$: $\alpha^* \approx 2.59$ versus $2.35$.
\end{enumerate}
\end{proposition}

\begin{proof}
\textit{Part (i).} The provisioning constraint gives
$\cos(v_j, c_i) < \tau_\text{safe} = \tau - \Delta < \tau$.

\textit{Part (ii).} A collision under the operating system
occurs iff $\cos(v_j, r') \geq \tau$; if
$\cos(v_j, r') < \tau$, no collision occurs regardless of
whether $\cos(v_j, r') \geq \tau_\text{safe}$.

\textit{Part (iii).} Monotonicity follows from
$d\alpha^*(0,\tau)/d\tau = -1/(\tau^2\sqrt{1-\tau^2}) < 0$,
so $\tau_\text{safe} < \tau$ implies
$\alpha^*(0, \tau_\text{safe}) > \alpha^*(0, \tau)$.
\end{proof}

\begin{remark}
The safety buffer reduces the available GV capacity bound.
Since $\mu(\cdot, \tilde{d})$ is \emph{strictly decreasing}
in $\tau$:
\begin{equation}
  \tau_\text{safe} < \tau
  \;\implies\;
  \mu(\tau_\text{safe}, \tilde{d}) \;\geq\; \mu(\tau, \tilde{d})
  \;\implies\;
  \frac{1}{\mu(\tau_\text{safe}, \tilde{d})}
  \;\leq\;
  \frac{1}{\mu(\tau, \tilde{d})},
\end{equation}
so $A_\text{GV}(\tau_\text{safe}) \leq A_\text{GV}(\tau)$.
The capacity cost depends strongly on the choice of $\Delta$:
at $\tau_\text{safe}{=}0.360$, $\tau{=}0.391$
($\Delta{=}0.031$, one operating step):
$A_\text{GV}\geq1.81\times10^9\approx2^{30.8}$,
giving headroom $\approx1{,}810\times$ over $|\mathcal{V}|{=}10^6$,
a comfortable margin.
By contrast, at $\tau_\text{safe}{=}0.319$
($\Delta{=}0.072$, two operating steps below):
$A_\text{GV}\geq2.38\times10^7\approx2^{24.5}$,
leaving only $\approx24\times$ headroom over one million
identities, a tight constraint that may limit provisioning
at scale.
Practitioners should choose $\Delta$ based on the expected
enrollment growth rate.
\end{remark}

\Paragraph{Lifecycle Policy.}
When a collision occurs, the newly enrolled real identity is never
modified: $r'$ enters $\mathcal{R}$ unconditionally and correction
falls entirely on the virtual side.
Each enrollment batch $\Delta\mathcal{R}$ triggers one screening pass
over $\mathcal{V}$, which costs
$O(|\mathcal{V}| \times |\Delta\mathcal{R}|)$ cosine operations in
the exact case and is reducible with approximate nearest-neighbor
indexing~\cite{johnson2021faiss}, with three outcomes.
Identities whose similarity to every $r' \in \Delta\mathcal{R}$ stays
below $\tau_\text{safe}$ are left untouched;
those in the monitoring zone $[\tau_\text{safe}, \tau)$ of
Proposition~\ref{prop:buffer} are queued for reassessment and
reassigned where needed;
only those reaching $\tau$ require immediate revocation.
Only the affected identities are reconsidered: the gallery and the
remaining virtual population are never re-provisioned.
A collision between two real identities admits no such remedy, since
neither person can be re-provisioned.

\subsection{Empirical Effective Capacity Estimator}
\label{app:empirical}

\Paragraph{Distinction from Theoretical Capacity.}
$A_\text{GV}(\tilde{d}, \tau)$ is a theoretical packing lower
bound under the spherical submanifold approximation.
The empirical estimator below measures a distinct quantity:
\begin{equation}
  A_{\mathrm{eff}}(\tau)
  := \frac{1}{p_{\mathrm{coll}}(\tau)},
  \qquad
  p_{\mathrm{coll}}(\tau)
  = \Pr\!\bigl[\cos(v_j, r') \geq \tau\bigr],
\end{equation}
the inverse empirical collision probability for a real identity
$r'$ drawn from $\mathcal{M}_\text{face}$.
Under the idealized uniform $\mathbb{S}^{\tilde{d}-1}$ model,
$p_\text{coll} = \mu(\tau, \tilde{d})$, so
$A_\text{eff} = A_\text{GV}$ numerically, but the two are
different mathematical objects:
$A_\text{GV}$ lower-bounds the packing number
$A_\text{pack}$ of the model sphere, while $A_\text{eff}$
characterises the empirical collision density of real identities
relative to the provisioned set.

\Paragraph{Statistical Model.}
Let $v_j \in \mathcal{V}$ ($j = 1, \ldots, N$) be provisioned
virtual identities and let
$r'_l \in \mathcal{R}_\text{test}$ ($l = 1, \ldots, L$) be
unseen real identities drawn independently from
$\mathcal{M}_\text{face}$.
Define the collision indicator
$X_{jl} = \mathbf{1}[\cos(v_j, r'_l) \geq \tau]$,
with $\Pr[X_{jl}=1] = 1/A_\text{eff}$ under the uniform
distribution assumption.
The total collision count
$C = \sum_{j=1}^{N}\sum_{l=1}^{L} X_{jl}$
satisfies:
\begin{equation}
  C \;\sim\;
  \mathrm{Poisson}\!\left(\frac{NL}{A_\mathrm{eff}}\right),
  \label{eq:poisson_model}
\end{equation}
%
by the Poisson limit theorem for sums of rare independent
events, where the individual event probability
$p_\mathrm{coll}=1/A_\mathrm{eff}$ is small; the expected count
$\lambda=NL/A_\mathrm{eff}$ need not be much smaller than one.

\Paragraph{Maximum Likelihood Estimator.}
The log-likelihood of model~\eqref{eq:poisson_model} is:
\begin{equation}
  \ell(A)
  = -\frac{NL}{A} + C\log\frac{NL}{A} - \log C!\,.
\end{equation}
Setting $d\ell/dA = 0$ yields:
\begin{equation}
  \hat{A}_{\mathrm{eff}}^{\mathrm{MLE}}
  = \frac{NL}{C}
  = \frac{|\mathcal{V}| \times |\mathcal{R}_\text{test}|}{C}.
  \label{eq:mle}
\end{equation}

\Paragraph{Confidence Interval.}
An exact $(1 - \alpha_\text{CI})$ confidence interval for
$A_\text{eff}$ is obtained by inverting the exact Poisson
confidence interval for $C$.
With $\chi^2_{\nu,q}$ denoting the $q$-th quantile of the
chi-squared distribution with $\nu$ degrees of freedom:
\begin{equation}
  \left[
    \frac{NL}{\dfrac{1}{2}\chi^2_{2(C+1),\;1-\alpha_\text{CI}/2}},\;
    \frac{NL}{\dfrac{1}{2}\chi^2_{2C,\;\alpha_\text{CI}/2}}
  \right].
  \label{eq:ci}
\end{equation}
This interval maintains the stated coverage probability
under the Poisson model regardless of sample size.

\Paragraph{Zero-Collision Lower Bound.}
When $C=0$, the MLE~\eqref{eq:mle} is undefined.
From $P(C=0)=e^{-NL/A_\text{eff}}\geq0.05$:
\begin{equation}
  \hat{A}_\text{eff}>\frac{NL}{\ln20}\approx\frac{NL}{2.996}.
  \label{eq:zero_bound}
\end{equation}
For $|\mathcal{V}|=1\text{M}$, $|\mathcal{R}_\text{test}|=180\text{K}$:
\begin{equation}
  NL=1.8\times10^{11},
  \qquad
  \hat{A}_\text{eff}>6.0\times10^{10}\approx2^{35.8}.
\end{equation}
In this setting, the empirical lower bound lies below
$A_\text{GV}(269,0.391)\approx2^{36.1}$, consistent with the
interpretation of Eq.~\eqref{eq:zero_bound} as a confidence
bound conditioned on $C=0$, while $A_\text{GV}$ is a
model-based packing lower bound.

\Paragraph{Poisson Approximation Validity.}
The Poisson approximation requires individual collision
probabilities $\mu(\tau,\tilde{d})\ll1$, not necessarily
$\lambda=NL\cdot\mu\ll1$.
Across all six operating points, the cap probability remains small:
the largest value occurs at
$\tau{=}0.319$, where $\mu=4.21\times10^{-8}$; at the primary
operating point $\tau{=}0.391$, $\mu=1.35\times10^{-11}$; and at $\tau{=}0.448$,
$\mu=4.92\times10^{-15}$.
These small individual collision probabilities support the Poisson
approximation under the rare-event model.
The expected total collision count $\lambda$ varies substantially
across operating points. For $|\mathcal{V}|{=}1$M and
$|\mathcal{R}_\text{test}|{=}180$K, giving
$NL=1.8\times10^{11}$:
\begin{equation}
  \lambda(0.319)
  = 1.8\times10^{11}\times4.21\times10^{-8}
  \approx 7.58\times10^3,
\end{equation}
\begin{equation}
  \lambda(0.391)
  = 1.8\times10^{11}\times1.35\times10^{-11}
  \approx 2.43,
\end{equation}
\begin{equation}
  \lambda(0.448)
  = 1.8\times10^{11}\times4.92\times10^{-15}
  \approx 8.9\times10^{-4}.
\end{equation}
At $\tau{=}0.319$, $\lambda\gg1$, so zero-collision bounds
are not meaningful and the MLE Eq.~\eqref{eq:mle} should be
used directly.
At $\tau{=}0.391$, $\lambda\approx2.43$: $C{=}0$ with
probability $\approx9\%$, $C{=}1$ with $\approx21\%$, and
$C{\geq}2$ with $\approx70\%$; the MLE applies in the majority
of cases.
At $\tau{=}0.448$, $\lambda\ll1$ and $C{=}0$ is near-certain
($>99.9\%$), so the zero-collision bound Eq.~\eqref{eq:zero_bound}
will typically apply.
We verify stability of $\hat{A}_\text{eff}$ as
$|\mathcal{R}_\text{test}|$ grows from 36K to 180K
(Sec.~\ref{sec:open-world}); flat curves confirm that estimates
are not dominated by local demographic density effects.

\Paragraph{Assumption Audit.}
The estimator~\eqref{eq:mle} rests on two assumptions:

\textbf{(i) Uniform distribution on $\mathcal{M}_\text{face}$.}
Real identities cluster by demographic factors (ethnicity,
age, gender), so $r'_l$ is not strictly uniform over
$\mathcal{M}_\text{face}$~\cite{su2026localscore}.
If $\mathcal{V}$ resides in a demographically sparse region,
$\hat{A}_\text{eff}$ is biased upward.
The stability check described above detects this.

\textbf{(ii) Independence of collision events.}
$X_{jl}$ and $X_{j'l}$ for different $j, j'$ at fixed $l$
are not exactly independent when $v_j$ and $v_{j'}$ are
close.
Under the Poisson limit theorem, the approximation
in~\eqref{eq:poisson_model} holds under weak dependence
when $p_\text{coll} \ll 1$, which is satisfied here.

\subsection{Relationship between Capacity, \texorpdfstring{$\alpha$}{alpha}, and
PCA Noise}\label{app:alpha_kappa}

The GV bound $A_\text{GV}$ and the model capacity
$A_\text{pack}$ are determined by $(\tau, \tilde{d})$ alone
and are independent of the proposal parameters $\alpha$ and
$\kappa$.
These parameters control how efficiently the algorithm
proposes candidates that pass the hard BIP checks.

Define the valid region at step $t$:
\begin{equation}
  \mathcal{V}_t^*
  = \bigl\{
    v \in \mathbb{S}^{d-1} :
    \cos(v, c_i) < \tau\;\forall c_i \in \mathcal{R},\;
    \cos(v, v_j) < \tau\;\forall v_j \in \mathcal{V}_t
  \bigr\},
\end{equation}
and the acceptance probability at step $t$:
\begin{equation}
  P_\text{accept}^{(t)}(\alpha,\kappa)
  = \Pr\!\Bigl[
    \mathrm{normalize}(r + \alpha z) \in \mathcal{V}_t^*
  \Bigr].
\end{equation}
Then:
\begin{equation}
  \mathbb{E}[|\mathcal{V}|]
  = \sum_{t=1}^{N_\text{attempts}} P_\text{accept}^{(t)}(\alpha,\kappa)
  \;\approx\;
  \bar{P}_\text{accept}(\alpha,\kappa)
  \cdot N_\text{attempts},
  \qquad
  \mathbb{E}[|\mathcal{V}|] \leq A_{\mathrm{pack}},
  \label{eq:fundamental}
\end{equation}
where $\bar{P}_\text{accept}$ is the average acceptance
probability over the provisioning process.

\textbf{$\kappa = 0$ (pure repulsion):}
$z = z^*$ provides directional bias away from the
neighborhood centroid but yields low diversity across
repeated candidates sampled from the same reference $r$.

\textbf{$\kappa\to\infty$ (pure PCA noise):}
Under a crude disjoint-cap and uniform-density approximation:
\begin{equation}
  P_\text{accept}^{(t)}
  \approx
  1-(M+|\mathcal{V}_t|)\,\mu(\tau,\tilde{d})
  = 1-\frac{M+|\mathcal{V}_t|}{A_\text{GV}}.
  \label{eq:paccept_limit}
\end{equation}
At $\tau=0.391$ and $|\mathcal{V}|=10^6$:
\begin{equation}
  1-\frac{360{,}232+10^6}{7.41\times10^{10}}
  \approx 99.998\%.
\end{equation}
For the 10M embedding-allocation experiment:
\begin{equation}
  1-\frac{360{,}232+10^7}{7.41\times10^{10}}
  \approx 99.986\%.
\end{equation}
The acceptance probability remains near unity throughout
provisioning at both scales.

PCA-aware weighting by $\sigma_k$ \emph{biases} perturbations
toward high-variance identity directions and suppresses
low-variance directions, keeping candidates close to the
empirically dominant face-identity subspace.
The sum in Eq.~\eqref{eq:perturbation} runs over all
$d = 512$ components, so perturbations are not strictly
confined to the top-$\tilde{d}$ subspace, but the contribution from low-variance directions is substantially
down-weighted by $\sigma_k$, although the residual variance
beyond the top-$\tilde{d}$ components is not exactly zero.

\section{Additional Implementation Details of \texorpdfstring{$G$}{G} (GapGen)}
\label{app:gapgen}

\begin{figure}[t]
  \centering
  \includegraphics[width=\linewidth]{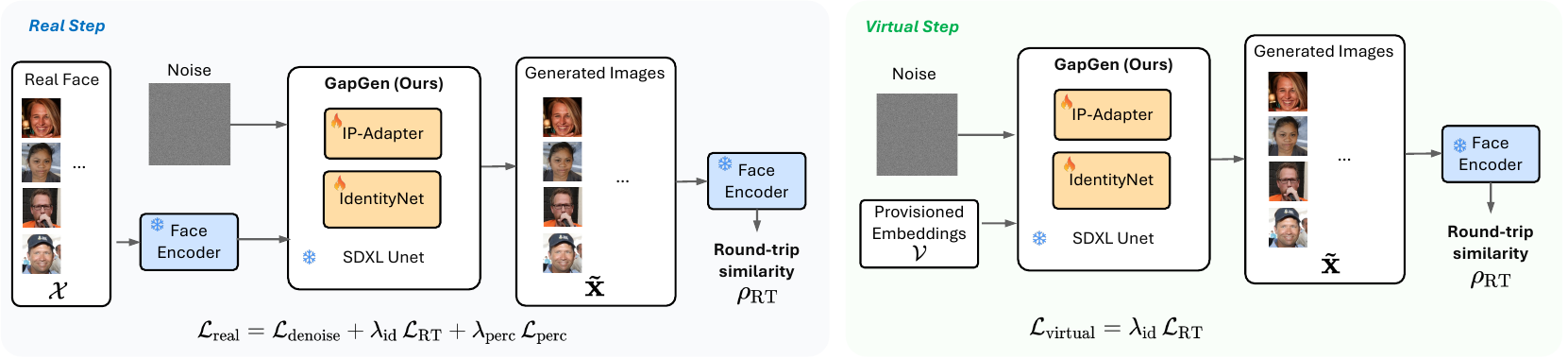}
  \caption{\textbf{GapGen training pipeline.} 
  Real and virtual steps are interleaved at mixing ratio $\beta$ during 
  fine-tuning. \textbf{Real step} (left): real face $x \in \mathcal{X}$ 
  is encoded by the frozen face encoder $\phi$ to produce target 
  $e^* = \phi(x)$; GapGen renders $\tilde{x} = G(e^*)$, supervised by 
  $\mathcal{L}_\text{denoise} + \lambda_\text{id}\mathcal{L}_\text{RT} 
  + \lambda_\text{perc}\mathcal{L}_\text{perc}$.
  \textbf{Virtual step} (right): a provisioned embedding 
  $s \in \mathcal{V}$ is fed directly as condition with no paired 
  ground-truth image; supervision relies solely on the round-trip 
  identity loss $\mathcal{L}_\text{RT} = 1 - \cos(\phi(\tilde{x}), s)$.
  IP-Adapter cross-attention projections and IdentityNet 
  are jointly fine-tuned; SDXL UNet and the face encoder remain frozen throughout.}
  \label{fig:gapgen_arch}
\end{figure}

Fig.~\ref{fig:gapgen_arch} summarizes the GapGen training pipeline. 
Below we describe the base pipeline, conditioning, sampling, round-trip 
re-encoding, and the gap-aware fine-tuning curriculum in turn.

\Paragraph{Base pipeline.}
The SDXL~\cite{podell2024sdxl} UNet and text encoders are frozen; the two
identity-aware components, the IP-Adapter~\cite{ye2023ip} cross-attention
projections and the IdentityNet, are jointly fine-tuned.

\Paragraph{Conditioning.}
Because BIP allocates and renders identities in the same ArcFace
embedding space, a provisioned
$s \in \mathcal{V} \subset \mathbb{R}^{512}$ is fed directly through the
pipeline's \texttt{image\_embeds} input \emph{without} any face image,
face detector, or projection MLP on the condition side. We use a fixed
canonical 5-point keypoint layout (face scale $0.28$ of the canvas,
$y$-offset $-5\%$) as the IdentityNet condition for every identity.

\Paragraph{Sampling.}
All images are generated at $1024 \times 1024$ with
$\text{steps}{=}30$, $\text{guidance\_scale}{=}3.0$,
$\text{controlnet\_scale}{=}0.8$, $\text{ip\_adapter\_scale}{=}0.8$, and
the default SDXL scheduler. For virtual step, we use a single fixed prompt
(\textit{``candid color portrait photo of a person, natural lighting''})
and negative prompt suppressing CGI/cartoon styles, watermarks,
oversmooth skin, and low-resolution artifacts.

\Paragraph{Round-trip re-encoding.}
We re-extract embeddings from
$\tilde{x}=G(s)$ with the same frozen ArcFace pipeline used during
allocation: \texttt{antelopev2} face detection followed by
\texttt{glintr100} (IResNet-100). The re-encoded
$\tilde{e}=\phi(\tilde{x})$ is directly comparable to the target
$s$ via cosine similarity since both lie on the same 512-d unit sphere.

\Paragraph{Fine-tuning.}
We initialize the IP-Adapter from the pretrained InstantID
checkpoint and IdentityNet from the pretrained ControlNet~\cite{zhang2023controlnet}, then jointly
fine-tune both on a gap-aware curriculum
(Sec.~\ref{sec:realization}).
SDXL UNet and text encoders remain frozen throughout.




\section{Verification, Visual Examples, and IAPCT}
\label{app:iapct}

\subsection{v-LFW Protocol and Visual Examples}
\label{app:vlfw}

\Paragraph{Construction.}
We construct \textbf{v-LFW} as a structural counterpart to LFW: it
inherits the identity count and the official $10$-fold $\times\,600$-pair
split, but populates every identity slot with a freshly allocated
virtual identity $s_n \in \mathcal{V}$ (distinct from any real LFW
person by BIP non-collision), rendered as $\tilde{x}_n = G(s_n)$ at
$1024{\times}1024$ following the conditioning in
App.~\ref{app:gapgen}.

\textbf{Protocols.} The shared pair list is then reused under the standard real-image baseline (R-R) and the four v-LFW protocols:
\begin{itemize}
  \item \textbf{R-R}: standard LFW verification on original real images. 
  \item \textbf{V-V}: same protocol as R-R but on v-LFW images at the 
        matched slot positions, so the same/different structure is 
        preserved over virtual identities.
  \item \textbf{R-V}: one image taken from LFW, the other from v-LFW at 
        the same slot index. Every R-V pair crosses both identity and 
        reality by construction, so its score distribution is a non-mate 
        distribution. We report FAR at the threshold calibrated on R-R 
        same-identity scores at TAR\,=\,95\%.
  \item \textbf{Detection}: per-image binary classification of real 
        vs.\ virtual without identity labels, evaluated on the union of 
        LFW and v-LFW images. We report AUC of the IAPCT forensics score 
        $\hat{p}(x)$ (Sec.~\ref{app:iapct-arch}).
  \item \textbf{Unified}: joint recognition and detection on the combined 
        LFW + v-LFW pair set. A pair is unified-correct iff verification 
        and the per-image real-vs-virtual decisions are all correct; 
        the formal indicator is given in Eq.~(\ref{eq:unified}) below.
\end{itemize}

\Paragraph{Visual examples.}

Fig.~\ref{fig:intra_class} shows multiple images rendered from the same 
virtual identity $s \in \mathcal{V}$ under varying pose configurations, 
lighting conditions, and text style prompts. 
Each row corresponds to one virtual identity; columns show three independent generations sharing the same conditioning embedding $s$ but differing in keypoint layout, prompt, and sampling seed. Identity is preserved across columns within each row, while pose, expression, lighting, and background vary, supporting the within-identity variation required for v-LFW (Tab.~\ref{tab:vlfw}).

\begin{figure}[t]
  \centering
  \includegraphics[width=\linewidth]{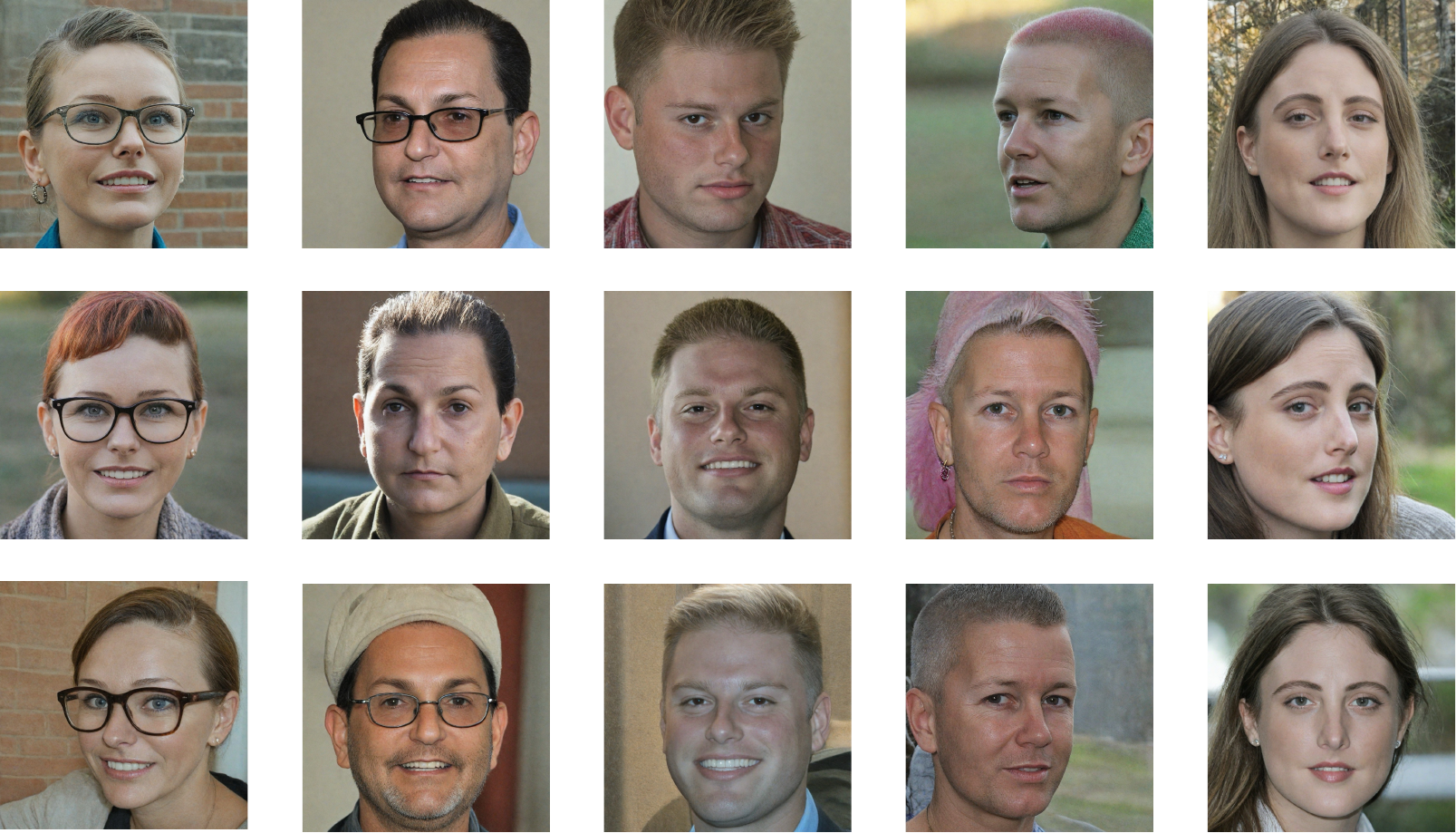}
  \caption{\textbf{v-LFW visual examples.} Each row shows three images 
  rendered from a single virtual identity $s \in \mathcal{V}$ at 
  $\alpha{=}4$, varying pose, lighting, and style while preserving 
  identity.}
  \label{fig:intra_class}
\end{figure}

\subsection{IAPCT: Identity-Anchored Patch Consistency Transformer}
\label{app:iapct-arch}

\Paragraph{Design.}
IAPCT augments the frozen ArcFace backbone with a transformer
head that interrogates whether local patch statistics at each
intermediate layer are consistent with the global identity
embedding $e = \phi(x)$.
The recognition path is preserved verbatim; the forensics
path is additive and introduces no overhead at inference
beyond the transformer forward pass.

\Paragraph{Backbone and intermediate features.}
We use IResNet-100 with pretrained ArcFace weights.
For an aligned input $x \in \mathbb{R}^{3\times112\times112}$,
spatial feature maps are tapped at four intermediate stages:
$\{F_l\}_{l=1}^{4}$, with spatial resolutions
$56{\times}56$, $28{\times}28$, $14{\times}14$, $7{\times}7$
and channel widths $64$, $128$, $256$, $512$ respectively.
$e_5 = \phi(x) \in \mathbb{R}^{512}$ is the standard
L2-normalized ArcFace identity output, left unchanged for
downstream recognition.

\Paragraph{Tokenisation.}
For each stage $l$, all spatial positions $p$ of $F_l$ are
extracted as patch tokens and projected to
$d_\text{model}{=}256$ via a per-stage linear layer:
\begin{equation}
  t_{l,p} = W_l\, F_l[:,p] \in \mathbb{R}^{256},
  \quad p = 1,\ldots,H_l W_l.
\end{equation}
The identity anchor is projected as:
$e_\text{proj} = W_\text{id}\, e \in \mathbb{R}^{256}$.

\Paragraph{Transformer encoder.}
A 4-layer transformer encoder processes the patch tokens
$\{t_{l,p}\}$ from all four stages jointly.
Each layer applies self-attention among spatial tokens
followed by cross-attention to the identity anchor
$e_\text{proj}$, producing patch-identity consistency scores:
\begin{equation}
  \gamma_{l,p} = \sigma\!\left(
  \frac{(W_Q\,t_{l,p}) \cdot (W_K\,e_\text{proj})}{\sqrt{d_k}}
  \right) \in (0,1),
\end{equation}
where $\sigma$ is the sigmoid function.
Real faces yield uniformly high $\gamma_{l,p}$ since every
spatial region originates from the same physical person;
virtual faces yield heterogeneous $\gamma_{l,p}$ due to
unconstrained intermediate features.

\Paragraph{Head and loss.}
A \texttt{[CLS]} token produces a forensics logit through:
\[
\mathrm{LN}(256)
\to \mathrm{Linear}(256{\to}128) \to \mathrm{GELU}
\to \mathrm{Dropout}(0.1) \to \mathrm{Linear}(128{\to}1).
\]
Training loss:
$\mathcal{L} = \mathrm{BCE}(\hat{y}, y) +
\lambda_c(H(\gamma^\text{virtual}) - H(\gamma^\text{real}))$,
where $H(\gamma^{(x)}) = -\sum_{l,p}\bar{\gamma}_{l,p}^{(x)}
\log\bar{\gamma}_{l,p}^{(x)}$ is the entropy of the normalized
attention distribution, $y{=}0$ for real and $y{=}1$ for
virtual, and $\phi$ remains frozen throughout.

\Paragraph{Training.}
Real samples: Glint360K $112{\times}112$ aligned crops.
Virtual samples: $100$K images at $\alpha{=}4$ from GapGen.
Batch size $128$ with $30\%$ virtual fraction; weighted
random sampling. AdamW, lr $10^{-4}$, weight decay $10^{-2}$,
gradient clip $5$.

\Paragraph{Unified metric.}
A pair $(x_a,x_b)$ is unified-correct iff verification and the
two per-image real-vs-virtual decisions are all correct:
\begin{equation}\label{eq:unified}
\mathbf{1}_{\mathrm{uni}}=
\mathbf{1}[\hat{y}_{\mathrm{ver}}{=}y_{\mathrm{ver}}]\cdot
\mathbf{1}[\hat{y}^{a}_{\mathrm{rv}}{=}y^{a}_{\mathrm{rv}}]\cdot
\mathbf{1}[\hat{y}^{b}_{\mathrm{rv}}{=}y^{b}_{\mathrm{rv}}],
\end{equation}
with $\hat{y}_{\mathrm{ver}}=\mathbf{1}[\cos(e_5(x_a),e_5(x_b))\geq\tau_{\mathrm{ver}}]$
at the per-fold LFW threshold and
$\hat{y}^{\bullet}_{\mathrm{rv}}=\mathbf{1}[\hat{p}(x_\bullet)\geq0.5]$.
Unified accuracy averages $\mathbf{1}_{\mathrm{uni}}$ over the
union of R-R, V-V, and R-V pairs. The joint indicator (rather
than an average over verification and detection accuracies)
matches the deployment scenario where a system must
simultaneously verify the identity \emph{and} flag the face as
real or virtual.

\section{Additional Results}
\label{app:add}

\subsection{t-SNE of \texorpdfstring{$\mathcal{R}$ vs.\ $\mathcal{V}$}{R vs. V}}
\label{app:add-tsne}

Fig.~\ref{fig:tsne} shows a t-SNE projection of real identity centroids
$\mathcal{R}$ (rendered as a gray density background) together with
virtual identities $\mathcal{V}$ at $\alpha\!\in\!\{2,5\}$ ($1$K samples
per panel). Each point in $\mathcal{V}$ is coloured by
$\max_{c_j\in\mathcal{R}}\cos(s,c_j)$, i.e., its cosine similarity to the
nearest real centroid: red indicates virtual identities that lie close to
some real cluster, while blue indicates identities that fall into
low-density gaps between clusters. Two observations are consistent with
Obs.~\ref{prop:manifold}. (i) $\mathcal{V}$ interleaves with
$\mathcal{R}$ rather than forming a separate cloud: virtual identities
appear throughout the manifold, including regions dense with real
centroids. (ii) Increasing $\alpha$ from $2$ to $5$ systematically pushes
$\mathcal{V}$ away from real clusters and into the gaps: the mean
max-cos decreases from $0.344$ ($\alpha{=}2$) to $0.288$ ($\alpha{=}5$),
consistent with the non-collision improvement reported in
Tab.~\ref{tab:alpha_threshold_main}. 

\begin{figure}[t]
  \centering
  \includegraphics[width=\linewidth]{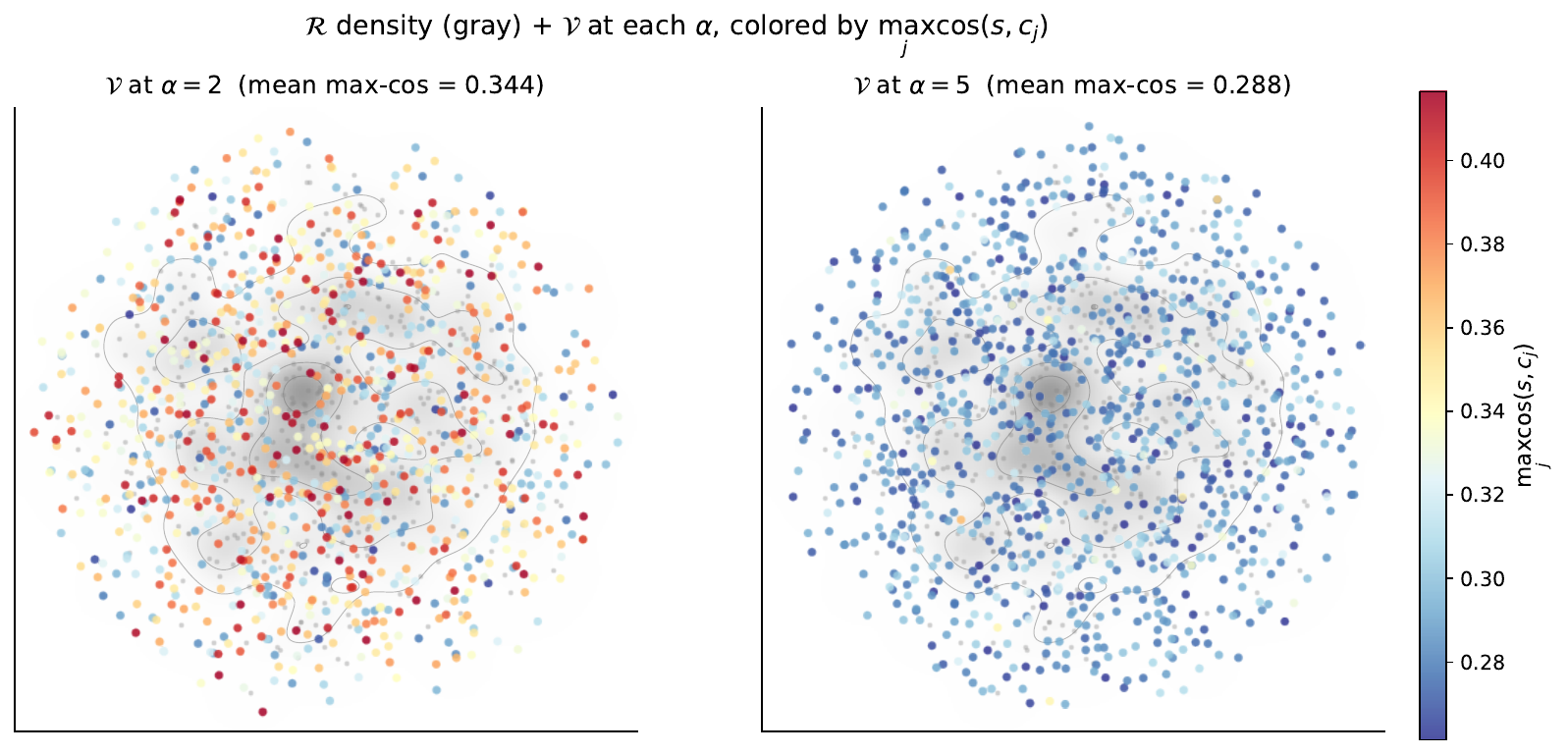}
\caption{\textbf{t-SNE of real and virtual identities.}
Real centroids $\mathcal{R}$ are shown as a gray density background.
Each virtual identity in $\mathcal{V}$ is coloured by
$\max_{c_j\in\mathcal{R}}\cos(s,c_j)$ (red: close to a real
cluster; blue: located in a low-density gap).
Increasing $\alpha$ shifts virtual identities from regions dense
with real centroids ($\alpha{=}2$, mean max-cos $0.344$) toward
low-density gaps ($\alpha{=}5$, mean max-cos $0.288$).
No internal collapse is observed within $\mathcal{V}$.}
  \label{fig:tsne}
\end{figure}

\subsection{Additional Image Grids}
\label{app:add-grid}

Fig.~\ref{fig:grid_extra} shows $40$ randomly selected faces generated
by $G$ from virtual identities $s \in \mathcal{V}$ at $\alpha{=}4$.
The samples vary visibly in apparent age, hairstyle, and accessories
(e.g.\ glasses) and show no obvious
identity duplication across images. This is qualitative evidence
that the perturbation-based provisioning scheme,
$s = \mathrm{normalize}(r + \alpha z)$, yields distinct identities
without collapse; it is not a measurement of demographic composition.

\begin{figure}[t]
  \centering
  \includegraphics[width=0.95\linewidth]{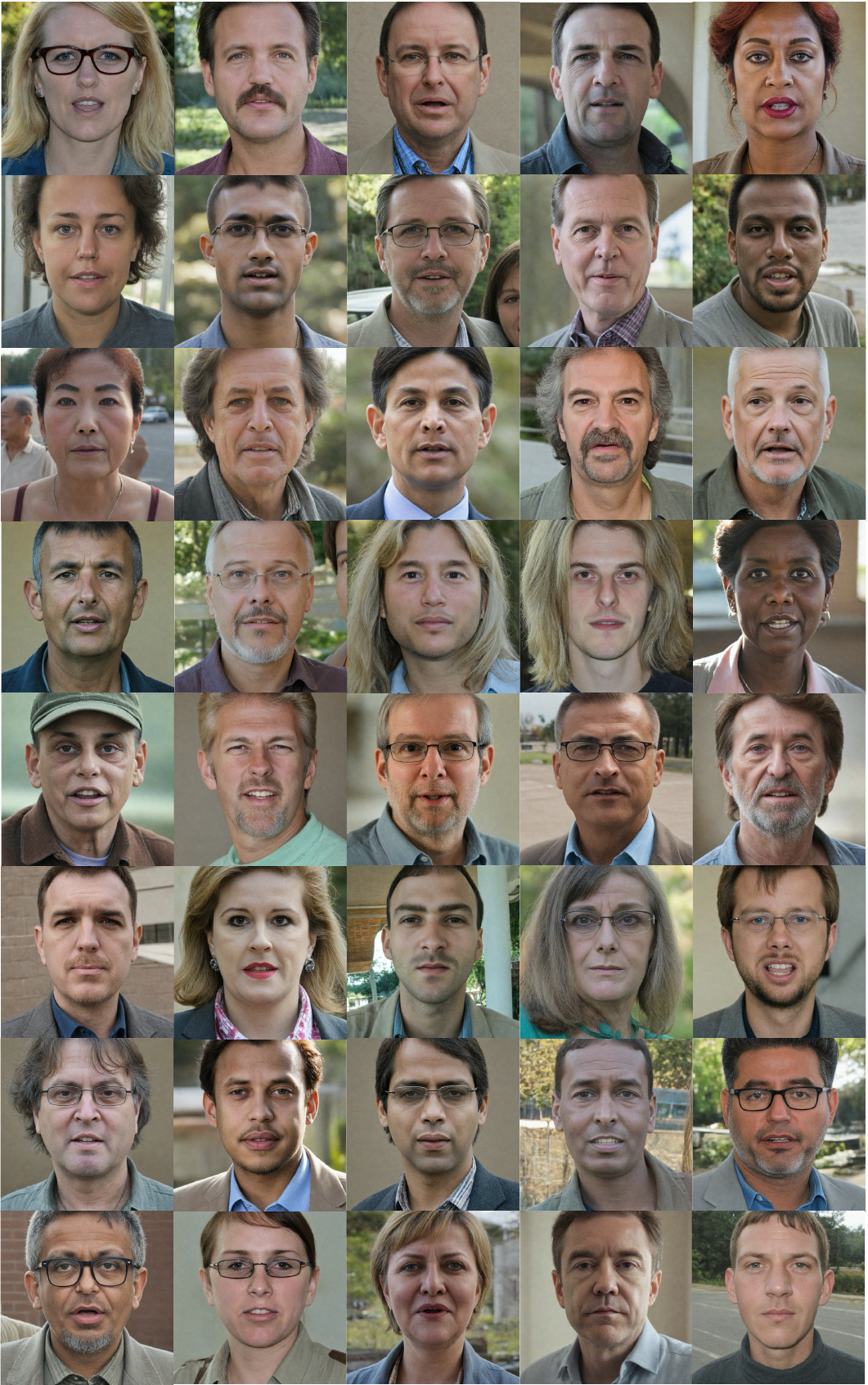}
  \caption{\textbf{Sample virtual identities from BIP.}
  Forty randomly sampled virtual identities rendered by GapGen
at $1024{\times}1024$.}
  \label{fig:grid_extra}
\end{figure}

Growing the gallery by half costs $0.82$ points, with no evident
acceleration.
The offset from $100\%$ at $0\%$ growth is a realization effect that
is present before any new enrollment (Tab.~\ref{tab:main}), not an
open-world one.
Relative to the identities that are valid at deployment, $0.83\%$
require reassignment after the 180K additions.

\subsection{Realized Capacity}
\label{app:realized_capacity}

Inter-Sep is computed per identity, so it decreases with the number
of identities $n$ even at constant pairwise quality.
Tab.~\ref{tab:realized_capacity} reports the largest scale at which
re-encoded Inter-Sep stays above a target, with real Glint360K
identities as a reference under the same encoder, threshold, and
protocol.

\begin{table}[t]
\centering
\small
\caption{Largest number of identities at which re-encoded Inter-Sep
stays above a target ($\tau{=}0.391$).}
\label{tab:realized_capacity}
\begin{tabular}{lccc}
\toprule
Inter-Sep target & BIP+GapGen, $\phi(G(s))$ & Glint360K (real) & Ratio \\
\midrule
99\% & 29{,}830  & 12{,}777  & $2.33\times$ \\
90\% & 300{,}132 & 141{,}248 & $2.12\times$ \\
\bottomrule
\end{tabular}
\end{table}

Under a rare-pair model, $S_n \approx (1-p)^{n-1}$ for a pairwise
collision probability $p$.
Inverting Tab.~\ref{tab:main} gives $p \approx 2$--$4\times10^{-7}$
across a $100\times$ range in $n$, with no systematic increase:
per-pair separability at 1M matches that at 10K, and the decline of
Inter-Sep from $99.57\%$ to $78.36\%$ is combinatorial, since each
identity must clear $100\times$ more neighbors.

\subsection{Encoder Transfer}
\label{app:encoder}

The guarantee is encoder-specific: Definition~\ref{def:bip} conditions
BIP on a given $\phi$ and $\tau$, and a different encoder induces a
different geometry.
Tab.~\ref{tab:encoder} re-encodes the same 10K realized identities and
the full 360K gallery with two AdaFace
IR-101~\cite{kim2022adaface} models, neither trained on Glint360K.
Each encoder is calibrated on its own IJB-B ROC at the same operating
point (FAR$\,{\approx}\,2\times10^{-5}$, Appx.~\ref{app:threshold}),
since cosine thresholds are not comparable across encoders.

\begin{table}[t]
\centering
\small
\caption{Encoder transfer on 10K realized identities. $\tau$ is the
threshold of each encoder at FAR$\,{\approx}\,2\times10^{-5}$ on IJB-B.}
\label{tab:encoder}
\begin{tabular}{llccc}
\toprule
Encoder & Training set & $\tau$ & Non-Collision (\%) & Inter-Sep (\%) \\
\midrule
ArcFace R100 (allocation encoder) & Glint360K  & 0.3910 & 98.91 & 99.57 \\
AdaFace IR-101 (unseen)           & WebFace4M  & 0.3870 & 99.59 & 99.68 \\
AdaFace IR-101 (unseen)           & WebFace12M & 0.3537 & 97.46 & 98.39 \\
\bottomrule
\end{tabular}
\end{table}

Both AdaFace models match or exceed ArcFace in verification accuracy
at this operating point (TAR of $94.37\%$ and $95.06\%$ against
$94.30\%$), so the comparison is not confounded by a weaker
recognizer.
Non-Collision ranges from $97.46\%$ to $99.59\%$ and Inter-Sep from
$98.39\%$ to $99.68\%$, on either side of the values of the allocation
encoder.
The worst case is $1.45$ points below ArcFace, at the model whose
calibrated threshold is lowest, which makes the non-collision
constraint the most demanding of the three in its own geometry.


\clearpage

\section*{NeurIPS Paper Checklist}

\begin{enumerate}

\item {\bf Claims}
    \item[] Question: Do the main claims made in the abstract and introduction accurately reflect the paper's contributions and scope?
    \item[] Answer: \answerYes{} 
    \item[] Justification: The abstract and introduction state the four
    contributions (formalization, geometry and allocation, GapGen
    realization, v-LFW benchmark) and their scope. Theoretical claims
    are supported by proofs in Appendix~\ref{app:capacity}; empirical
    claims are supported by
    Tabs.~\ref{tab:alpha_threshold_main}--\ref{tab:safety_vlfw},
    Figs.~\ref{fig:qualitative}--\ref{fig:openworld}, and
    Appendix~\ref{app:add}.

\item {\bf Limitations}
    \item[] Question: Does the paper discuss the limitations of the work performed by the authors?
    \item[] Answer: \answerYes{} 
    \item[] Justification: Limitations are discussed in the Limitations
    paragraph of Sec.~\ref{sec:conclu}: the gallery-relative guarantee
    and revocation under continued enrollment, the gap between
    allocation and realization, the spherical submanifold
    approximation, allocation in a single encoder space with transfer
    measured on two other encoders, robustness to physical capture as
    an open question, and the uncontrolled demographic composition of
    provisioned identities.

\item {\bf Theory assumptions and proofs}
    \item[] Question: For each theoretical result, does the paper provide the full set of assumptions and a complete (and correct) proof?
    \item[] Answer: \answerYes{} 
    \item[] Justification: All assumptions are stated explicitly
    (Assumption~\ref{ass:encoder}). Theorem~\ref{thm:gv},
    Proposition~\ref{prop:capacity}, Corollary~\ref{cor:alpha},
    and Proposition~\ref{prop:buffer} include complete proofs
    in Appendix~\ref{app:capacity}. Lemma~\ref{lem:displacement}
    follows by direct computation from the definition of $s$.

    \item {\bf Experimental result reproducibility}
    \item[] Question: Does the paper fully disclose all the information needed to reproduce the main experimental results of the paper to the extent that it affects the main claims and/or conclusions of the paper (regardless of whether the code and data are provided or not)?
    \item[] Answer: \answerYes{} 
    \item[] Justification: Full implementation details are provided in
    Sec.~\ref{sec:method}, Sec.~\ref{sec:exp} and the Appendix,
    including encoder, gallery, threshold, hyperparameters ($\alpha$,
    $\kappa$, $\beta$, $K$, $t$), and generation settings. The
    analyses after realization (realized
    capacity, encoder transfer) are specified in
    Appendix~\ref{app:add}. Vec2Face baselines are reproduced
    from the authors' released code.

\item {\bf Open access to data and code}
    \item[] Question: Does the paper provide open access to the data and code, with sufficient instructions to faithfully reproduce the main experimental results, as described in supplemental material?
    \item[] Answer: \answerNo{} 
    \item[] Justification: The GapGen model, the BIP allocation and
    evaluation code, the provisioned virtual identity embeddings, the
    generated virtual-face data, and the v-LFW dataset will be released
    through the project page
    (\url{https://vilab-group.com/project/bip}). All training data
    (Glint360K~\cite{an2021partial}, FFHQ~\cite{karras2019style})
    are publicly available. The base generator builds on
    InstantID~\cite{wang2024instantid}, which is publicly released.

\item {\bf Experimental setting/details}
    \item[] Question: Does the paper specify all the training and test details (e.g., data splits, hyperparameters, how they were chosen, type of optimizer) necessary to understand the results?
    \item[] Answer: \answerYes{} 
    \item[] Justification: Encoder, gallery, threshold, and all
    hyperparameters are specified in Sec.~\ref{sec:method},
    Sec.~\ref{sec:exp} and the Appendix. GapGen loss weights and mixing
    ratio are in Sec.~\ref{sec:exp}, its curriculum in
    Sec.~\ref{sec:realization}, and its conditioning and sampling
    settings in Appendix~\ref{app:gapgen}.  IAPCT architecture and training are in
    Appendix~\ref{app:iapct}.

\item {\bf Experiment statistical significance}
    \item[] Question: Does the paper report error bars suitably and correctly defined or other appropriate information about the statistical significance of the experiments?
    \item[] Answer: \answerNo{} 
    \item[] Justification: Non-Collision and Inter-Sep are deterministic
    given the provisioned set and threshold; there is no randomness
    across runs for these metrics. FID and the RT pass rate are
    computed over the full provisioned set. Error bars are
    not reported as the provisioning process is deterministic
    conditioned on the random seed, and running multiple full-scale
    (1M-identity) experiments is computationally prohibitive.

\item {\bf Experiments compute resources}
    \item[] Question: For each experiment, does the paper provide sufficient information on the computer resources (type of compute workers, memory, time of execution) needed to reproduce the experiments?
    \item[] Answer: \answerYes{} 
    \item[] Justification: All experiments use NVIDIA A100 GPUs.
    Embedding-space allocation of 1M identities requires
    approximately 1 GPU-hour. GapGen fine-tuning requires
    approximately 48 GPU-hours.

\item {\bf Code of ethics}
    \item[] Question: Does the research conducted in the paper conform, in every respect, with the NeurIPS Code of Ethics \url{https://neurips.cc/public/EthicsGuidelines}?
    \item[] Answer: \answerYes{} 
    \item[] Justification: The paper has been reviewed against the NeurIPS
    Code of Ethics. All training data are used under their respective
    licenses. v-LFW contains no real person's likeness. Broader impact
    and potential misuse are discussed in
    Appendix~\ref{app:broader}.

\item {\bf Broader impacts}
    \item[] Question: Does the paper discuss both potential positive societal impacts and negative societal impacts of the work performed?
    \item[] Answer: \answerYes{} 
    \item[] Justification: This work can support safe biometric identity
    infrastructure for digital entities. A potential risk is misuse
    for identity spoofing, which is mitigated by our gallery-relative
    provisioning design and by the deepfake detectability of virtual
    faces reported in the paragraph \textit{Deepfake Detectability of
    Virtual Faces} in Sec.~\ref{sec:exp} and further discussed in
    Appendix~\ref{app:broader}.

\item {\bf Safeguards}
    \item[] Question: Does the paper describe safeguards that have been put in place for responsible release of data or models that have a high risk for misuse (e.g., pre-trained language models, image generators, or scraped datasets)?
    \item[] Answer: \answerYes{} 
    \item[] Justification: The paragraph \textit{Deepfake Detectability of
    Virtual Faces} in Sec.~\ref{sec:exp} demonstrates that
    provisioned virtual faces are detectable by existing SoTA deepfake
    detectors zero-shot. The gallery-relative design requires access
    to the full enrollment gallery to provision non-colliding
    identities, limiting misuse. v-LFW will be released with usage
    restrictions.

\item {\bf Licenses for existing assets}
    \item[] Question: Are the creators or original owners of assets (e.g., code, data, models), used in the paper, properly credited and are the license and terms of use explicitly mentioned and properly respected?
    \item[] Answer: \answerYes{} 
    \item[] Justification: All assets are cited.
    Glint360K~\cite{an2021partial} is used under its research license.
    FFHQ~\cite{karras2019style} is used under CC-BY-NC-SA 4.0.
    InstantID~\cite{wang2024instantid} is used under Apache 2.0.
    ArcFace~\cite{deng2019arcface} model weights are from the publicly
    released \texttt{antelopev2} package.
    The other assets (LFW, IJB-B, WebFace4M, CASIA-WebFace, SDXL, the
    AdaFace models, the baseline generators DCFace, Arc2Face and
    Vec2Face, and the deepfake detectors) are cited where they are
    used and were obtained under the terms set by their providers.

\item {\bf New assets}
    \item[] Question: Are new assets introduced in the paper well documented and is the documentation provided alongside the assets?
    \item[] Answer: \answerYes{} 
    \item[] Justification: v-LFW is introduced as a new benchmark.
    Construction details are in Appendix~\ref{app:vlfw}. v-LFW
    contains only BIP-provisioned virtual identities with no real
    person's likeness. The dataset will be released with
    documentation and usage guidelines upon publication.

\item {\bf Crowdsourcing and research with human subjects}
    \item[] Question: For crowdsourcing experiments and research with human subjects, does the paper include the full text of instructions given to participants and screenshots, if applicable, as well as details about compensation (if any)?
    \item[] Answer: \answerNA{} 
    \item[] Justification: The paper does not involve crowdsourcing or
    research with human subjects. All evaluation is performed on
    existing datasets or algorithmically generated virtual identities.

\item {\bf Institutional review board (IRB) approvals or equivalent for research with human subjects}
    \item[] Question: Does the paper describe potential risks incurred by study participants, whether such risks were disclosed to the subjects, and whether Institutional Review Board (IRB) approvals (or an equivalent approval/review based on the requirements of your country or institution) were obtained?
    \item[] Answer: \answerNA{} 
    \item[] Justification: No human subjects are involved. All face images
    used for training are from publicly available datasets with
    appropriate research licenses.

\item {\bf Declaration of LLM usage}
    \item[] Question: Does the paper describe the usage of LLMs if it is an important, original, or non-standard component of the core methods in this research? Note that if the LLM is used only for writing, editing, or formatting purposes and does \emph{not} impact the core methodology, scientific rigor, or originality of the research, declaration is not required.
    \item[] Answer: \answerNA{} 
    \item[] Justification: LLMs are not used as a component of the core
    methodology. LLMs were used only for writing assistance and do not
    affect the scientific contributions, methodology, or results.

\end{enumerate}

\end{document}